\documentclass{article}

\usepackage{ArXiv_style}
\usepackage[british]{babel}

\usepackage{natbib} % has a nice set of citation styles and commands
    \renewcommand{\bibsection}{\subsubsection*{References}}
\usepackage{mathtools} % amsmath with fixes and additions
\usepackage{booktabs} % commands to create good-looking tables
\usepackage{tikz} % nice language for creating drawings and diagrams

\usepackage[utf8]{inputenc} % allow utf-8 input
\usepackage[T1]{fontenc}    % use 8-bit T1 fontsm
\usepackage{hyperref}       % hyperlinks
\hypersetup{hidelinks}
\usepackage{url}            % simple URL typesetting
\usepackage{booktabs}       % professional-quality tables
\usepackage{amsfonts}       % blackboard math symbols
\usepackage{nicefrac}       % compact symbols for 1/2, etc.

\usepackage{float}
\usepackage{multirow}
\usepackage{courier}
\usepackage{listings, lstautogobble,amsfonts}
\usepackage{amsmath,amssymb,amsthm}
\usepackage{mdframed}
\usepackage{mathtools}
\usepackage[finalizecache=false,frozencache=true,newfloat]{minted}
\usepackage{textcmds}
\usepackage{xspace}
\usepackage{xcolor}
\usepackage[normalem]{ulem}
\usepackage{caption}
\usepackage{multicol}
\usepackage{bbm}
\usepackage{thmtools}
\usepackage{bm}
\usepackage{thm-restate}
\usepackage[inline]{enumitem}
\usepackage{soul}
\usepackage{physics}
\usepackage{wrapfig}

\theoremstyle{definition}
\newtheorem{definition}{Definition}[section]
\theoremstyle{definition}

\theoremstyle{definition}
\newtheorem{example}{Example}[section]
\theoremstyle{definition}

\theoremstyle{definition}

\newenvironment{problogcode}{\captionsetup{type=listing}}{}
\newcommand{\dspl}{DeepSeaProbLog\xspace}
\newcommand{\dpl}{DeepProbLog\xspace}

\newcommand{\Neuralparams}{\ensuremath{\bm{\Lambda}}}

\newcommand{\neuralparam}{\ensuremath{\lambda}}

\newcommand{\reparam}{\ensuremath{r(\variables{u},\Neuralparams)}}

\newcommand{\deriv}{\ensuremath{\partial_\neuralparam}}

\newcommand{\dfacts}{\ensuremath{\mathcal{F}_D}}
\newcommand{\lrules}{\ensuremath{\mathcal{R}_L}}
\newcommand{\compset}{\ensuremath{\mathcal{C}_M}}
\newcommand{\compsubset}{\ensuremath{C_M}}
\newcommand{\negcompsubset}{\ensuremath{\overline{C}_M}}

\newcommand{\dsplprogram}{\ensuremath{\mathbb{P}}\xspace}
\newcommand{\world}{\ensuremath{\omega}\xspace}
\newcommand{\coolness}{\ensuremath{\beta}\xspace}

\newcommand{\weight}{\ensuremath{w\xspace}}
\newcommand{\variables}[1]{\ensuremath{\bm{#1}}\xspace}

\newcommand{\amc}{\ensuremath{\text{SP}}\xspace}

\newcommand{\relaxation}{\ensuremath{s}}
\newcommand{\softened}{\ensuremath{s}}

\declaretheoremstyle[%
  spaceabove=-6pt,%
  spacebelow=6pt,%
  headfont=\normalfont\itshape,%
  postheadspace=1em,%
  qed=\qedsymbol%
]{mystyle} 
\declaretheorem[name={Proof},style=mystyle,unnumbered,
]{prf}

\newcommand{\indicator}{\ensuremath{\mathbbm{1}}}

\definecolor{red_salsa}{HTML}{F94144}
\definecolor{orange_red}{HTML}{F3722C}
\definecolor{yellow_orange}{HTML}{F8961E}
\definecolor{mango_tango}{HTML}{F9844A}
\definecolor{maize_crayola}{HTML}{F9C74F}
\definecolor{pistachio}{HTML}{90BE6D}
\definecolor{jungle_green}{HTML}{43AA8B}
\definecolor{steel_teal}{HTML}{4D908E}
\definecolor{queen_blue}{HTML}{577590}
\definecolor{celadon_blue}{HTML}{277DA1}

\setminted[problog.py:ProbLogLexer -x]{
xleftmargin=0pt,
linenos=false,
escapeinside=@@,
mathescape=true,
breaklines=true,
fontfamily=courier,
breakautoindent=True
fontsize=\small
}

\newcounter{lstlabelcounter}

\newminted[problog]{problog.py:ProbLogLexer -x}{}
\newcommand{\probloginline}[1]{\mintinline{problog.py:ProbLogLexer -x}{#1}}

\SetupFloatingEnvironment{listing}{name=Listing}

\setminted[text]{
xleftmargin=20pt,
linenos=false,
mathescape=true,
escapeinside=@@,
breaklines=true,
}
\newminted[dcplp]{text}{}

\newcommand{\expone}{{\bf (E1)}\xspace}
\newcommand{\exptwo}{{\bf (E2)}\xspace}

\newcommand{\inlineimage}[1]{
\ensuremath{\vcenter{\hbox{\includegraphics[height=9pt]{#1}}}}}

\usepackage{microtype}
\usepackage{graphicx}
\usepackage{subfigure}
\usepackage{booktabs} % for professional tables
\title{Neural Probabilistic Logic Programming in Discrete-Continuous Domains}

\author{
  Lennert De Smet \\
  Dept. of Computer Science, \\
  KU Leuven \\
   \And
  Pedro Zuidberg Dos Martires \\
  Dept. of Computer Science, \\
  \"Orebro University \\
  \And
  Robin Manhaeve \\
  Dept. of Computer Science, \\
  KU Leuven \\
  \And
  Giuseppe Marra \\
  Dept. of Computer Science, \\
  KU Leuven \\
  \And
  Angelika Kimmig \\
  Dept. of Computer Science, \\
  KU Leuven \\
  \And
  Luc De Raedt \\
  Dept. of Computer Science, \\
  KU Leuven \\
}

\begin{document}
\maketitle

\begin{abstract}
    Neural-symbolic AI (NeSy) allows neural networks to exploit symbolic background knowledge in the form of logic. It has been shown to aid learning in the limited data regime and to facilitate inference on out-of-distribution data.
    Probabilistic NeSy focuses on integrating neural networks with both logic and probability theory, which additionally allows learning under uncertainty.
    A major limitation of current probabilistic NeSy systems, such as \dpl, is their restriction to finite probability distributions, i.e., discrete random variables.
    In contrast, deep probabilistic programming (DPP) excels in modelling and optimising continuous probability distributions.
    Hence, we introduce \dspl, a neural probabilistic logic programming language that incorporates DPP techniques into NeSy.
    Doing so results in the support of inference and learning of both discrete and continuous probability distributions under logical constraints.
    Our main contributions are
    1) the semantics of \dspl and its corresponding inference algorithm,
    2) a proven asymptotically unbiased learning algorithm, and
    3) a series of experiments that illustrate the versatility of our approach.
\end{abstract}

\section{Introduction}\label{sec:intro}

Neural-symbolic AI (NeSy)~\citep{garcez2002neural,dereadt2021statistical} focuses on the integration of symbolic and neural methods. 
The advantage of NeSy is that it combines the reasoning power of logical representations with the learning capabilities of neural networks.
Additionally, it has been shown to converge faster during learning and to be more robust~\citep{rocktaschel2017end,xu2018semantic,evans2018learning}.

The challenge of NeSy lies in combining discrete symbols with continuous and differentiable neural representations.
So far, such a combination has been realised for Boolean variables by interpreting the outputs of neural networks as the weights of these variables.
% So far this has been accomplished by interpreting the outputs of neural networks as the weights of Boolean variables. 
These weights can then be given either a fuzzy semantics~\citep{badreddine2022logic,diligenti2017sbr}
or a probabilistic semantics~\citep{manhaeve2021neural,yang2020neurasp}. 
The latter is also used in neural probabilistic logic programming (NPLP), where neural networks parametrise probabilistic logic programs.

A shortcoming of traditional probabilistic NeSy approaches is that they fail to capture models that integrate continuous random variables and neural networks --
a feature already achieved with mixture density networks~\citep{bishop1994mixture} and more generally deep probabilistic programming (DPP)~\citep{tran2017deep,bingham2019pyro}.
% Despite the expressiveness of these methods, they have so far focused on efficient probabilistic inference in continuous domains, e.g., Hamiltonian Monte Carlo or variational inference.
However, it is unclear whether DPP can be generalised to enable logical and relational reasoning.
Hence, a gap exists between DPP and NeSy as reasoning is, after all, a fundamental component of the latter.
We contribute towards closing this DPP-NeSy gap by introducing \dspl\footnote{\q{Sea} stands for the letter C, as in {\bf c}ontinuous random variable.},
an NPLP language
with support for discrete-continuous random variables that retains logical and relational reasoning capabilities.
We achieve this integration by allowing
% neural networks to parametrise 
arbitrary and differentiable probability distributions expressed in a modern DPP language
while combining knowledge compilation~\citep{darwiche2002knowledge} with the reparametrisation trick~\citep{ruiz2016generalized} and continuous relaxations~\citep{petersen2021learning}.
% More concretely, we allow for neural networks to parameterize arbitrary and differentiable probability distributions.
% We achieve this using the reparameterization trick~\citep{ruiz2016generalized} and continuous relaxations~\citep{petersen2021learning}. This stands in contrast to \dpl~\citep{manhaeve2018deepproblog} where only finite categorical distributions are supported.

Our main contributions are \begin{enumerate*}[label=(\arabic*)]
    \item the well-defined probabilistic semantics of \dspl (Section~\ref{sec:dspl}) with an inference algorithm based on
    % a state-of-the-art 
    weighted model integration (WMI)~\citep{belle2015probabilistic}
    % technique
    (Section~\ref{subsec:reduction}),
    \item a proven asymptotically unbiased gradient estimate for WMI that turns \dspl into a differentiable, discrete-continuous NPLP language (Section~\ref{subsec:learning}), and
    % \item the well-defined probabilistic semantics of \dspl, a differentiable  discrete-continuous NPLP language\label{contrib1} (Section~\ref{sec:dspl}),
    % \item an implementation of inference and gradient-based learning algorithms utilising knowledge compilation and DPP techniques (Section~\ref{sec:learning}), and\label{contrib2}
    \item an experimental evaluation showing the versatility of discrete-continuous reasoning and the efficacy of our approach\label{contrib3} (Section~\ref{sec:experiments})
\end{enumerate*}.

\section{Logic programming concepts}
\label{sec:lp}

A term~\probloginline{t} is either a constant~\probloginline{c}, a variable~\probloginline{V} or a structured term of the form
\probloginline{f(t@$_1$@,...,t@$_K$@)},
where \probloginline{f} is a functor and each \probloginline{t@$_i$@} is a term. 
Atoms are expressions of the form
\probloginline{q(t@$_1$@,...,t@$_K$@)}.
Here, \probloginline{q/@$K$@} is a predicate of arity $K$ %(of arity $n$, or $q/n$ in shorthand notation)
and each
\probloginline{t@$_i$@} is a term.
A literal is an atom or the negation of an atom 
$\neg \text{\probloginline{q(t@$_1$@,...,t@$_K$@)}}$.
A definite clause (also called a rule) is an expression of the form
\probloginline{h:- b@$_1$@,...,b@$_K$@}
where \probloginline{h} is an atom and each \probloginline{b@$_i$@} is a literal.
Within the context of a rule, \probloginline{h} is called the head and the conjunction of \probloginline{b@$_i$@}'s is referred to as the body of the rule. Rules with an empty body are called facts.
A logic program is a finite set of definite clauses. 
If an expression does not contain any variables, it is called ground.
Ground expressions are obtained from non-ground ones by means of substitution.
A substitution
$\theta = \{
    \text{\probloginline{V@$_1$@}}
    =
    \text{\probloginline{t@$_1$@}}
    ,
    \dots,
    \text{\probloginline{V@$_K$@}}
    =
    \text{\probloginline{t@$_K$@}}
\}$
is a mapping from variables~\probloginline{V@$_i$@} to terms~\probloginline{t@$_i$@}. 
Applying a substitution $\theta$ to an expression \probloginline{e} (denoted \probloginline{e@$\theta$@})
replaces each occurrence of \probloginline{V@$_i$@} in \probloginline{e} with the corresponding \probloginline{t@$_i$@}.

While {\em pure} Prolog (or definite clause logic) is defined using the concepts above, practical implementations of Prolog extend definite clause logic with  an external arithmetic engine~\cite[Section 8]{sterling1994art}. Such engines enable the use of system specific routines in order to handle numeric data efficiently.
Analogous to standard terms in definite clause logic, as defined above, we introduce numeric terms. A numeric term  \probloginline{n@$_i$@} is either a numeric constant (a real, an integer, a float, etc.), a numeric variable \probloginline{N@$_i$@},
or a numerical functional term, which is an expression of the form
\probloginline{@$\varphi$@(n@$_1$@,...,n@$_K$@)} where \probloginline{@$\varphi$@} is an externally defined numerical function.
The difference between a standard logical term and a numerical term is that {\em ground} numerical terms are evaluated and yield a numeric constant. For instance,  if \probloginline{add} is a function, then 
\probloginline{add(3, add(5, 0))} evaluates to the numerical constant \probloginline{8}. 

Lastly, numeric constants can be compared to each other using a built-in binary comparison operator
$\bowtie\ \in
\left\{
\text{\probloginline{<}},
\text{\probloginline{=<}},
\text{\probloginline{>}},
\text{\probloginline{>=}},
\text{\probloginline{=:=}},
\text{\probloginline{=\=}}
\right\}$.
Here we use Prolog syntax to write comparison operators, which correspond to $\{ <, \leq, >, \geq, =, \neq \}$ in standard mathematical notation.
Comparison operators appear in the body of a rule, have two arguments, and are generally written as
$
\text{\probloginline{@$\varphi_l$@(n@$_{l,1}$@,...,n@$_{n,K}$@)}}
\bowtie
\text{\probloginline{@$\varphi_r$@(n@$_{r,1}$@,...,n@$_{r,K}$@)}}.
$
They evaluate
their left and right side and subsequently compare the results, assuming everything is ground.
If the stated comparison holds, the comparison is interpreted by the logic program as true, else as false.

\section{\dspl}
\label{sec:dspl}

\subsection{Syntax}
While facts in pure Prolog are deterministically true, in probabilistic logic programs they are annotated with the probability with which they are true.
These are the so-called probabilistic facts~\citep{de2007problog}. When working in discrete-continuous domains, we need to use the more general concept of distributional facts~\citep{zuidberg2020atoms},
inspired by the distributional clauses of~\citet{gutmann2011magic}.

\begin{definition}[Distributional fact]  \emph{Distributional facts} are expressions of the form
\probloginline{x ~ distribution(n@$_1$@,...,n@$_K$@)}, where \probloginline{x} denotes a term,  the \probloginline{n@$_i$@}'s are numerical terms and
\probloginline{distribution} expresses the probability distribution according to which \probloginline{x} is distributed.
\end{definition}

\begin{example}[Distributional fact]
To declare a Poisson distributed variable \probloginline{x} with rate parameter $\lambda$, one would write \probloginline{x ~ poisson(@$\lambda$@).}
\end{example}

The meaning of a distributional fact is that all ground instances $\text{\probloginline{x}}\theta$ serve as random variables that are distributed according to $\text{\probloginline{distribution(n@$_1\theta$@,...,n@$_K\theta$@)}}$.
To obtain a neural-symbolic interface, we will allow neural networks to parametrise these distributions.

\begin{definition}[Neural distributional fact]
\label{def:ndf}
A \emph{neural distributional fact} (NDF) is a distributional fact \probloginline{x ~ distribution(n@$_1$@,...,n@$_K$@)} in which
a subset
% $\{ \text{\probloginline{f@$_j$@}}\}_{j=1}^{L} $ 
of the set of numerical terms $\{ \text{\probloginline{n@$_i$@}}\}_{i=1}^{K}$ is implemented by neural networks that depend on a set of parameters $\Neuralparams$.
\end{definition}

% The random variables defined by the neural distributional facts enter the logic by expressing that only part of their domain leads to satisfiability. Intuitively, such expressions can be interpreted as constraints.
Random variables defined by NDFs can then be used in the logic in the form of comparisons, e.g., $\probloginline{x} > \probloginline{y}$, to reason about desired ranges of the variables.

\begin{definition}(Probabilistic comparison formula)
\label{def:pcf}
A \emph{probabilistic comparison formula} (PCF) is an expression of the form $(g(\boldsymbol{x}) \bowtie 0)$,
where $g$ is a function applied to the set of random variables $\variables{x}$ and $\bowtie\ \in
\left\{
\text{\probloginline{<}},
\text{\probloginline{=<}},
\text{\probloginline{>}},
\text{\probloginline{>=}},
\text{\probloginline{=:=}},
\text{\probloginline{=\=}}
\right\}$ is a binary comparison operator.
A PCF is called \emph{valid} if $\left\{\boldsymbol{x}\ |\ g(\boldsymbol{x}) \bowtie 0\right\}$ is a \emph{measurable} set.
\end{definition}

\begin{example}[Probabilistic comparison formula]
If \probloginline{x} is Poisson distributed and represents the number of chocolate pieces put in a chocolate biscuit, then we can use a simple PCF to define when such a biscuit passes a quality test through the rule \probloginline{passes_test :- (x > 11).}
% Importantly, we would like to both compute the probability of \probloginline{passes_test} and compare it to a learning signal.
\end{example}

Note that the general form of a PCF in Definition~\ref{def:pcf} has a $0$ on the right hand side, which can always be obtained by subtracting the right hand-side from both sides of the relation.
With the definitions of NDFs and PCFs, a \dspl program can now be formally defined.

\begin{definition}[\dspl program]
A \emph{\dspl program} consists of
a finite set of  NDFs $\dfacts$ (defining random variables),
a finite set $\compset$ of  valid PCFs
and a set of logical rules $\lrules$ that can use any of those valid PCFs in their bodies.
\end{definition}

\begin{example}[\dspl program]
\label{ex:dsplprogram}
% Consider the \dspl program below.
\probloginline{humid} denotes a Bernoulli
random variable that takes the value $1$ with a probability $p$ given by the output of a neural network \probloginline{humid_detector}. \probloginline{temp} denotes a normally distributed variable whose parameters are predicted by a network \probloginline{temperature_predictor}.
The program further contains two rules that deduce whether we have good weather or not. The first one expresses the case of snowy weather, while the second holds for a rather temperate and dry situation.
The atom \probloginline{query(good_weather(}$\inlineimage{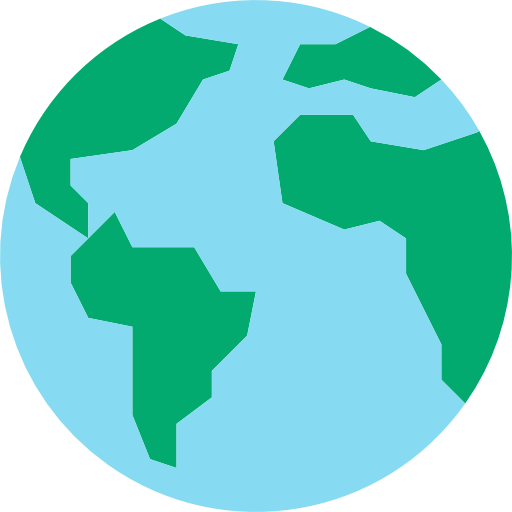}$\probloginline{))} declares that we want to compute the probability of \probloginline{good_weather} when evaluated on the data $\inlineimage{Imagery/world.png}$.
% The query atom at the end declares the probability of the atom we would like to compute and also tells us which ground term to replace the logic variables with.
It illustrates the neural-symbolic nature of \dspl, as its ground argument is a sub-symbolic representation ($\inlineimage{Imagery/world.png}$) of the world.

{\footnotesize
\begin{problog}
humid(Data) ~ 
    bernoulli(humid_detector(Data)).
temp(Data, T) ~ 
    normal(temperature_predictor(Data)).

good_weather(Data):- 
    humid(Data) =:= 1, temp(Data) < 0.

good_weather(Data):- 
    humid(Data) =:= 0, temp(Data) > 15.

query(good_weather(@$\inlineimage{Imagery/world.png}$@)).
\end{problog}
}
\end{example}

Notice how the random variables \probloginline{humid} and \probloginline{temp} appear in the body of a logical rule with comparison operators.
In our probabilistic setting, the truth value of a comparison depends on the value of its random variables and is thus random itself.
% Furthermore, to obtain well-defined probability distributions, we need to restrict these comparison operators to being Lebesgue-measurable. 

\paragraph{\dspl generalises a range of existing PLP languages.}
% \dspl generalises a range of existing PLP languages.
For instance, if we were to remove the distributional fact on \probloginline{temp} and all the PCFs using them, we would obtain a \dpl program~\citep{manhaeve2021neural}. If we additionally replace the neural network in \probloginline{humid} with a fixed probability \probloginline{p}, we end up with a probabilistic logic program~\citep{de2007problog}. Replacing that constant probability \probloginline{p} by a constant \probloginline{1} yields a non-probabilistic Prolog program.
Alternatively, considering all rules and facts in Example~\ref{ex:dsplprogram} but
replacing the neural parameters of the normal distribution with numeric constants results in a Distributional Clause program~\citep{gutmann2011magic}.
We further discuss these connections
% of \dspl to these languages further
in Appendix~\ref{app:special}, where we also formally prove that \dspl strictly generalises \dpl.

\subsection{Semantics}
\dspl programs are used to
% answer probabilistic queries of ground atoms, i.e., to 
compute the probability that a ground atom \probloginline{q} is entailed.
That probability follows from the semantics of a \dspl program.
As is custom in (probabilistic) logic programming, we will define the semantics of \dspl with respect to ground programs.
We will assume that each ground distributional fact $f \in \dfacts$ defines a different random variable, as each random variable can only have one unique distribution.
Also notice that any ground neural distributional facts will contain the inputs to their neural functions.
In a sense, a \dspl program is conditioned on these neural network inputs.

To define the semantics of ground \dspl programs, we first introduce the possible worlds over the PCFs.
Every subset $\compsubset$ of a set of PCFs $\compset$ defines a possible world $\world_{\compsubset} = \{ \compsubset \cup h \theta \mid \lrules \cup \compsubset  \models h\theta \text{ and } h\theta \text{ is ground} \}$.
Intuitively speaking, the comparisons in such a subset are considered to be true and all others false. A rule with a comparison in its body that is not in this subset can hence not be used to determine the truth value of atoms. The deterministic rules $\lrules$ and the subset $\compsubset$ together define a set of all ground atoms $h\theta$ that are derivable, i.e., entailed by the program, and thus considered true. Such a set is called a \textit{possible world}. We refer the reader to the paper of~\citet{deraedt2015concepts} for a detailed account of possible worlds in a PLP context.
Following the distribution semantics of~\citet{sato1995statistical} and by taking inspiration from~\citet{gutmann2011magic}, we define the probability of a possible world.

\begin{definition}[Probability of a possible world]
Let \dsplprogram be a ground \dspl program
and $\compsubset=\{ c_1,\dots, c_H \} \subseteq \compset$ a set of PCFs that depend on the random variables declared in the set of distributional facts $\dfacts$. The probability $P(\world_{\compsubset})$ of a world $\world_{\compsubset}$
is then defined as
\begin{align}
    % P(\world_{\compsubset}) =
        \int
        \left[
            \Big(\prod_{c_i\in \compsubset} \indicator(c_i) \Big)
            \Big( \prod_{c_i \in \compset\setminus \compsubset}  \indicator(\bar{c}_i) \Big)
        \right]
        \ \differential P_{\dfacts}.
        \label{eq:world_probability}
\end{align}

Here the symbol $\indicator$ denotes the indicator function, $\bar{c}_i$ is the complement of the comparison $c_i$ and $\differential P_{\dfacts}$ represents the joint probability measure of the random variables defined in the set of distributional facts $\dfacts$.
\end{definition}

\begin{example}[Probability of a possible world]
Given \dsplprogram as in Example \ref{ex:dsplprogram}, 
where \probloginline{humid_detector(data1)} predicts $p(\text{\probloginline{data1}})$ and \probloginline{temperature_predictor(data1)} predicts the tuple $(\mu(\text{\probloginline{data1}}),\sigma(\text{\probloginline{data1}}))$, the probability of the possible world $\omega_{\left\{\text{\probloginline{temp(data1)>15}},\ \text{\probloginline{humid(data1)=:=1}}\right\}}$ is given by
\begin{align}
    p(\text{\probloginline{data1}})\cdot
    \int\indicator(x {>} 15)
    \frac{
    \exp
    \left(
        -\frac{
            \left(x-\mu(\text{\probloginline{data1}}) \right)^2
        }
        {
            2 \sigma^2(\text{\probloginline{data1}})
        }
    \right)
    }{\sqrt{2\pi} \sigma(\text{\probloginline{data1}})}
    \ \differential x.
\end{align}
Indeed, the measure $\differential P_{\dfacts}$ decomposes into a counting measure and the product of a Gaussian density function with a differential. The counting measure leads to the factor $p(\text{\probloginline{data1}})$, since that is the probability that \probloginline{humid(data1)=:=1}. Hence, the products in Equation~\ref{eq:world_probability} reduce to a single indicator of the PCF $(x > 15)$.
\end{example}

\begin{definition}[Probability of query atom]
The probability of a ground atom $q$ is given by
\begin{align}
    P(q) = \sum_{\compsubset \subseteq \compset : q \in \world_{\compset} } P(\world_{\compsubset})
    \label{eq:query_probability}.
\end{align}
\end{definition}
\begin{restatable}[Measureability of query atom]{proposition}{querymeasurability}
\label{proposition:query_measurability}
Let \dsplprogram be a \dspl program, then
\dsplprogram 
defines, for an arbitrary query atom $q$, the probability that $q$ is true.
\label{prop:semantics}
\end{restatable}
\vspace{5.5pt}
\begin{prf}
See Appendix~\ref{app:proof_query_probability}.
\end{prf}

\section{Inference and learning}\label{sec:learning}

\subsection{Inference via weighted logic}
\label{subsec:reduction}

A popular technique to perform inference in probabilistic logic programming uses a reduction to so-called {\em weighted model counting} (WMC); instead of computing the probability of a query, one computes the weight of a propositional logical formula~\citep{chavira2008probabilistic,fierens2015inference}. For \dspl, the equivalent approach is to map a ground program onto a {\em satisfiability modulo theory} (SMT) formula~\citep{barrett2018satisfiability}. The analogous concept to WMC for these formulas is {\em weighted model integration} (WMI)~\citep{belle2015probabilistic}, which can handle infinite sample spaces. 
In all that follows, for ease of exposition, we assume that all joint probability distributions are continuous.

\begin{restatable}[Inference as WMI]{proposition}{inferenceaswmi}
\label{proposition:inferenceaswmi}
Assume that the measure $\differential P_{\dfacts}$ decomposes into a joint probability density function $\weight(\variables{x})$ and a differential $\differential \variables{x}$, then the probability $P(q)$ of a query atom $q$ can be expressed as the weighted model integration problem
\begin{align}
    % P(q) =
    \int \left[
        \sum_{\compsubset \subseteq \compset : q \in \world_{\compset}} \prod_{c_i\in \compsubset {\cup} \negcompsubset}  \indicator(c_i(\variables{x})) 
    \right]
    \weight(\variables{x})
    \ \mathrm{d}\variables{x},
    \label{eq:probability_query_as_wmi}
\end{align}
where
$
    \negcompsubset \coloneqq \left\{\bar{c}_i\ |\ c_i \in \compset {\setminus} \compsubset\right\}
$
.
\end{restatable}
\vspace{5.5pt}
\begin{prf}
See Appendix~\ref{app:proof_inference_as_wmi}.
\end{prf}
Being able to express the probability of a queried atom in \dspl as a weighted model integral allows us to adapt and deploy inference techniques developed in the weighted model integration literature for \dspl.
% to perform inference in \dspl.
We opt for the approximate inference algorithm \q{Sampo} presented in~\citet{zuidberg2019exact} because of its more scalable nature.
% at least with respect to variables with infinite sample spaces.
Sampo uses knowledge compilation~\citep{darwiche2002knowledge}, a state-of-the-art technique for probabilistic logic inference~\citep{chavira2008probabilistic,fierens2015inference}.
% to construct the integrand of Equation~\ref{eq:probability_query_as_wmi}.
Intuitively, knowledge compilation is a two-step procedure applied to a logical formula with PCFs, i.e., an SMT formula.
First, it infers the exact probability of all PCFs containing discrete variables through symbolic inference.
Then, it converts the remainder of the SMT formula into a polynomial in terms of those exact probabilities and the PCFs containing continuous random variables (Figure~\ref{fig:amcdiagram2}).
This polynomial is the integrand of Equation~\ref{eq:probability_query_as_wmi}. All that remains is to approximate the integration of this polynomial by sampling from the joint probability distribution $\weight(\variables{x})$ of the continuous random variables.
% Intuitively, knowledge compilation infers PCFs containing variables with finite sample spaces in an exact, symbolic way, meaning all assignments of those variables that satisfy the PCF are determined and their probabilities computed exactly.
% While knowledge compilating, PCFs containing variables with finite sample spaces are inferred in an exact, symbolic way, meaning all satisfying assignments of those variables and their probabilities are determined and computed exactly.
% PCFs containing continuous random variables are abstracted to binary random variables and still need to be integrated, which is approximated by sampling values from the joint probability distribution $w(\variables{x})$ of those variables.
In other words, Sampo computes the expression
% While doing so, it also performs exact, symbolic inference on PCFs containing variables with finite sample spaces, in an equivalent way to \dpl inference.
% PCFs containing continuous random variables then still need to be integrated, which is approximated by sampling values from the joint probability distribution $w(\variables{x})$ of those variables. In other words, Sampo computes the expression
\begin{align}
    P(q) =
    \int \amc(\variables{x}) \cdot  \weight(\variables{x}) \, \differential\variables{x}
   % \ \approx\ 
   \approx
    \frac{1}{|\mathcal{X}|}
    \sum_{\variables{x} \in \mathcal{X}} \amc(\variables{x}),
    \label{eq:probability_query_approx}
\end{align}
where $\mathcal{X}$ denotes a set of samples drawn from $\weight(\variables{x})$ and $\amc(\variables{x})$ is the result of knowledge compilation, i.e., the sum of products of indicator functions in Equation~\ref{eq:probability_query_as_wmi}.

\begin{figure}
    \centering
    \includegraphics[width=0.7\linewidth]{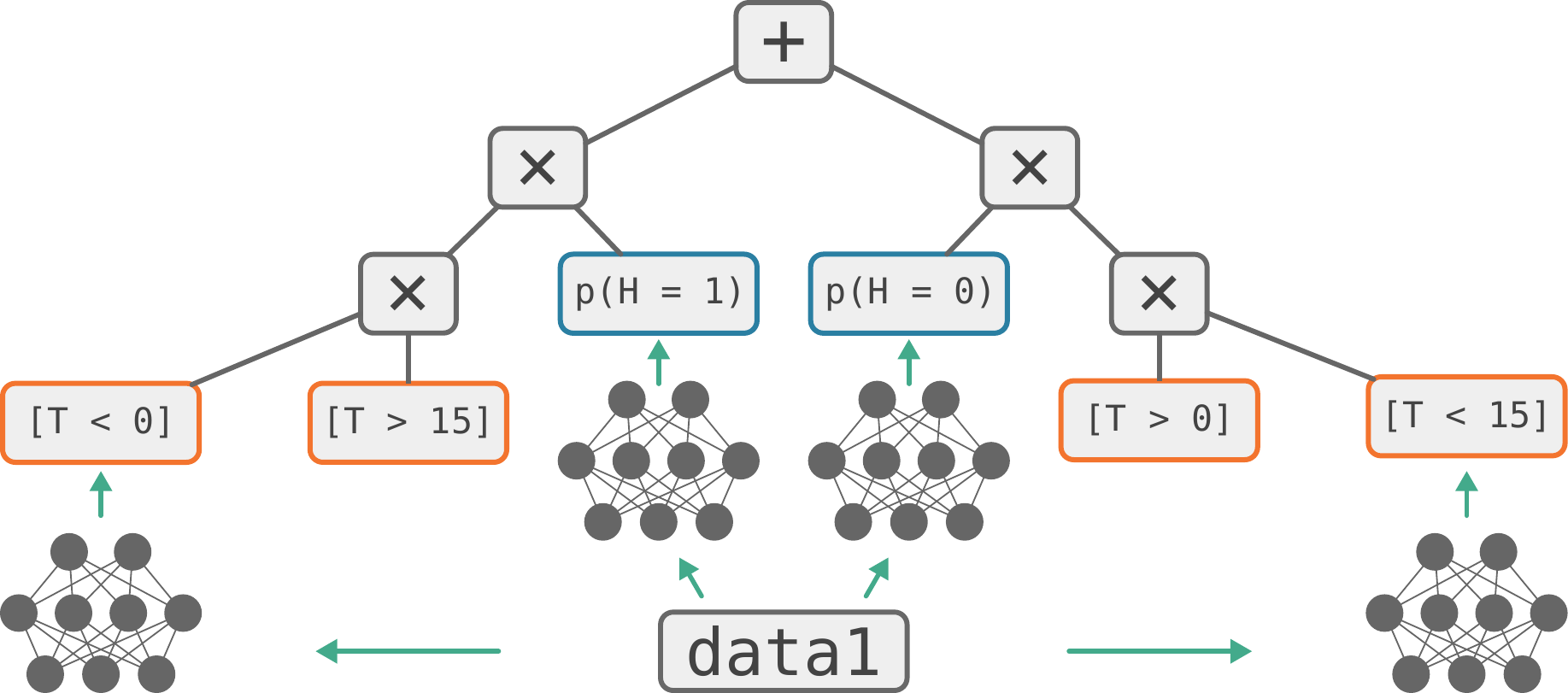}
    \caption{Diagrammatic representation of the result of knowledge compilation for the query in Example~\ref{ex:dsplprogram}.
    The \textcolor{celadon_blue}{blue} boxes originate from PCFs over discrete variables, while the \textcolor{orange_red}{orange} ones are PCFs over continuous variables. Note how the discrete variable PCFs are reduced to their exact probabilities while the continuous PCFs still need to be inferred.
    }
    \label{fig:amcdiagram2}
\end{figure}

We stress that the Sampo algorithm only samples random variables whose expected value with respect to the function $\amc(\variables{x})$ can not be computed exactly. 
Hence, in the absence of continuous random variables, our implementation of \dspl using Sampo coincides with \dpl on both a semantics level (Proposition~\ref{proposition:deepproblogspecial}) and inference level.
% Figure~\ref{fig:amcdiagram2} provides a diagrammatic representation of the function $\amc(\variables{x})$ for the query in Example~\ref{ex:dsplprogram} that illustrates the result of knowledge compilation and how exact symbolic inference for the discrete variables interacts with PCFs containing continuous variables.

\subsection{Learning via differentiation}\label{subsec:learning}

A \dspl program depend on a set of (neural) parameters
$\Neuralparams$ (Definition~\ref{def:ndf}).
In order to optimise these parameters, we need to take their gradients of a loss function that compares the probability $P(q)$ to a training signal. More precisely, we need to compute the derivative
\begin{align}
     \deriv \mathcal{L}(P_{\Neuralparams}(q)) =
%    \ \approx\ 
    \partial_{P_{\Neuralparams}(q)} \mathcal{L}(P_{\Neuralparams}(q))\cdot \deriv P_{\Neuralparams}(q), 
\end{align}
where we explicitly indicate the dependency of the probability on $\Neuralparams$ and $\neuralparam\in \Neuralparams$.
Differentiating $P_{\Neuralparams}(q)$ with respect to $\lambda$ presents two obstacles.
First,
the question of differentiating through the sampling  process of Equation~\ref{eq:probability_query_approx}
and second,
the non-differentiability of the indicator functions in $\amc(\variables{x})$.

The non-differentiability of sampling is tackled using the reparametrisation trick~\citep{ruiz2016generalized}.
Reparametrisation offers better estimates than other approaches, such as REINFORCE~\citep{williams1992simple} and is readily utilised in modern probabilistic programming languages such as Tensorflow Probability~\citep{tran2017deep} and Pyro~\citep{bingham2019pyro}.
Conversely, the non-differentiability of the indicator functions prevents swapping the order of differentiation and integration~\citep{flanders1973differentiation}, which we resolve by applying continuous relaxations following the work of~\citet{petersen2021learning}.
Together, we obtain the gradient estimate
\begin{align}
    \deriv P_{\Neuralparams}(q)
    &= 
    \deriv 
    \int \amc(\variables{x}) \cdot  \weight_{\Neuralparams}(\variables{x}) \, \differential\variables{x} \\
    % &\approx\ 
    % \deriv \int \amc_\softened(\reparam) \cdot
    % p(\variables{u})
    % \ \differential \variables{u} \\
    &\approx\ 
    \int
    \left[
        \deriv \amc_\softened(\reparam)
    \right] \cdot 
    p(\variables{u})
    \ \differential \variables{u}
    \label{eq:probability_derivative}
    ,
\end{align}
where the subscript $s$ in $\amc_\softened(\variables{x}) $ denotes the continuously relaxed or \q{softened} version of $\amc(\variables{x})$ and $r(\variables{u}, \Neuralparams)$ is the reparametrisation function.

\paragraph{Our gradient estimate using relaxations is asymptotically unbiased.}
As an example of these relaxations, consider the indicator of a PCF $(g(\boldsymbol{x}) > 0)$, which is relaxed into the sigmoid $\sigma(\coolness\cdot g(\boldsymbol{x}))$.
% In Appendix~\ref{sec:detaileddiff} we provide a more detailed account covering relaxations of general PCFs.
Appendix~\ref{app:detaileddiff} provides more details on relaxations of general PCFs.
The \textit{coolness} parameter $\coolness\in \mathbb{R}_+^0$
% used in the relaxation of the indicator of a PCF
% is the inverse of the temperature of the relaxation and 
determines the strictness of the relaxation. Hence, we recover the hard indicator function when $\coolness{\rightarrow} +\infty$.
Note that relaxing indicator functions 
introduces bias.
% results in a biased expression for $\deriv P_{\Neuralparams}(q)$.
\citet{petersen2021learning} already stated in their work that, in the infinite coolness limit, a relaxed function coincides with the non-relaxed one. Proposition~\ref{proposition:unbiasedapprox} extends this result
% by stating that this property also holds for 
to the derivatives of relaxed and non-relaxed functions, proving that our gradient estimate is asymptotically unbiased.

\begin{restatable}[Unbiased in the infinite coolness limit]{proposition}{unbiasedapprox}
\label{proposition:unbiasedapprox}
Let $\mathbb{P}$ be a \dspl program with PCFs $(g_i(\boldsymbol{x}) \bowtie 0)$ and corresponding coolness parameters $\coolness_i$. \\
If all $\deriv (g_i \circ r)$ are locally integrable over $\mathbb{R}^k$ and
every $\coolness_i \rightarrow +\infty$,
then we have, for any query atom $q$, that
\begin{align}
    \deriv P(q)
    =
    \int
    \deriv\amc_\softened(\reparam)
    \cdot p(\variables{u})
    \ \differential \variables{u}
    .
\end{align}
\end{restatable}
\begin{proof}
The proof makes use of the mathematical theory of distributions~\citep{schwartz1957theorie}, which generalise the concept of functions, and is given in Appendix~\ref{app:unbiasedness}.
\end{proof}

Finally, we obtain a practical and unbiased estimate of $\deriv P_{\Neuralparams}(q)$ using a set of samples $\mathcal{U}$ drawn from 
$p(\variables{u})$.
\begin{align}
    \deriv P(q)
    &\approx\ 
    \int
        \left[
            \deriv      \amc_\softened(\reparam) 
        \right]
        \cdot p(\variables{u})
    \ \differential \variables{u} \\
    &\approx\ 
    \frac{1}{|\mathcal{U}|} \sum_{\variables{u} \in \mathcal{U}} \deriv \amc_\softened(\reparam).
    \label{eq:WMIDerFinalApprox}
\end{align}
Computing this gradient estimate does not require drawing new samples. Implementing the relaxations of PCFs in a \q{straight-through} manner allows us to directly apply automatic differentiation on the inferred probability.

\subsection{Probabilistic programming connections}

% Since knowledge compilation allows us to only sample from a continuous joint probability density function,
Since knowledge compilation symbolically infers discrete random variables, we only have to sample from a continuous joint probability distribution.
To sample such distributions, we can fully exploit the advanced inference and learning techniques~\citep{hoffman2014no} of modern probabilistic programming languages~\citep{tran2017deep, bingham2019pyro}. Our implementation of \dspl utilises Tensorflow Probability for this task, effectively using knowledge compilation as a differentiable bridge between logical and probabilistic reasoning.
While this bridge is limited to sampling techniques for now, it presents an interesting direction for future work to completely unify NeSy with DPP.

\subsection{Limitations}\label{subsec:limitations}

While the use of relaxations is well-known and used in recent gradient estimators~\citep{tucker2017rebar, grathwohl2018backpropagation}, the bias they introduce is often hard to deal with in practice.
In our case, this bias only reduces to zero in the infinite coolness limit (Proposition~\ref{proposition:unbiasedapprox}), meaning the use of annealing can be necessary. Finding a good annealing scheme for any problem is non-trivial and effectively introduces another component in need of optimisation. However, as relaxations allow the use of the reparametrisation trick, the resulting lower variance estimates together with our theoretical guarantees support our choice. A more detailed discussion of the current limitations of \dspl can be found in Appendix~\ref{app:limitations}.
% Moreover, a single, simple annealing scheme proved to be sufficient for all of our experiments.

\section{Related work}
\label{sec:related}

From a NeSy perspective the formalism most closely related to \dspl is that of {\em Logic Tensor Networks} (LTNs)~\citep{badreddine2022logic}. The main difference between LTNs and \dspl is the fuzzy logic semantics of the former and the probabilistic semantics of the latter.
Interestingly, LTNs and other NeSy approaches based on fuzzy logic also require relaxations to incorporate continuous values.
% Interestingly, both systems use similar continuous relaxations when differentiating through comparisons of continuous variables, which is also in line with other NeSy approaches based on fuzzy logics~\citep{marra2019lyrics}. 
However, fuzzy-based approaches require these relaxations at the semantics level, in contrast to \dspl.
Even more, they can only compare continuous point values instead of more general continuous random variables.
LTNs' fuzzy semantics also exhibit drawbacks on a more practical level.
Unlike \dspl with its probabilistic semantics, LTNs are not inherently capable of neural-symbolic generative modelling (Section~\ref{subsec:vae}). 
For a broader overview of the field of neural-symbolic AI, we refer the reader to a series of survey papers that have been published in recent years~\citep{garcez2019neural, marra2021statistical,garcez2022neural}.

From a probabilistic programming perspective, \dspl is related to languages that handle discrete and continuous random variables such as \emph{BLOG}~\citep{milch2006probabilistic}, \emph{Distributional Clauses}~\citep{gutmann2011magic} and \emph{Anglican}~\citep{tolpin2016design}, which have all been given declarative semantics, i.e., the meaning of the program does not depend on the underlying inference algorithm. However, these languages have the drawback of non-differentiability.
This drawback stands in stark contrast to end-to-end (deep) probabilistic programming languages such as Pyro~\citep{bingham2019pyro} or Tensorflow Probability~\citep{dillon2017tensorflow}, but these have only been equipped with operational semantics and do not support logical constraints.
\dspl not only introduces the ability to express such logical constraints in the form of PCFs to construct challenging posterior distributions, but
% also maintains end-to-end differentiability.
does so in an end-to-end differentiable fashion.

Finally, our gradient estimate can be related to relaxation-based methods like REBAR~\citep{tucker2017rebar} or RELAX~\citep{grathwohl2018backpropagation}, but without the REINFORCE~\citep{williams1992simple} inspired component. Instead, we utilise the differentiability of knowledge compilation to obtain exact gradients of discrete variables.
Since our inference scheme innately requires knowledge compilation, the use of other discrete gradient estimators like~\citep{niepert2021implicit} does not directly apply to \dspl.
Moreover, we exploit the structure of our problem by directly relaxing comparison formulae in a sound manner~\citep{petersen2021learning}, in contrast to introducing an artificial relaxation of the whole problem~\citep{grathwohl2018backpropagation}.
% Since our inference scheme requires knowledge compilation, which automatically gives exact gradients for the discrete variables, the use of other purely discrete gradient estimators like~\citep{niepert2021implicit} does not apply to \dspl for now.
% Finally, \dspl can exploit advanced inference techniques already present in deep probabilistic programming language, such as \emph{stochastic variational inference}~\citep{hoffman2013stochastic} or \emph{NUTS Hamiltonian Monte Carlo}~\citep{hoffman2014no}, to efficiently sample high-dimensional distributions. In practice, our current implementation of \dspl does so by using Tensorflow Probability as its arithmetic engine in the back-end.

\section{Experimental Evaluation}\label{sec:experiments}

% We have two experimental questions to check whether \dspl is a comprehensive integration of DPP and logic.
% \qlearning Can we perform deep density estimation from distant supervision based on logical constraints?
% \qgap Can we tackle deep generative modelling conditioned on logical constraints?
% We have two main experimental questions. \qlearning Is learning, which includes inference, with continuous relaxations and reparametrisations possible? 
% \qgap Does \dspl bridge the DPP-NeSy gap?
% We answer \qlearning on the detection of handwritten dates (cf. Section~\ref{subsec:subtraction}) and a neural hybrid Bayesian net (cf. Section~\ref{subsec:hybridnet}). \qgap will be answered by introducing \emph{neural-symbolic variational auto-encoders}, inspired by the work of~\citet{misino2022vael}.
We illustrate the versatility of \dspl by tackling three different problems. Section~\ref{subsec:dates} discusses the detection of handwritten dates without location supervision. In Section~\ref{subsec:hybridnet} a hybrid Bayesian network with conditional probabilities dependent on the satisfaction of certain logical constraints will be optimised. Finally, Section~\ref{subsec:vae} introduces neural-symbolic variational auto-encoders, inspired by~\citet{misino2022vael}. 

The details of our experimental setup, including the precise \dspl programs, coolness annealing schemes and hyperparameters used for the neural networks are given in Appendix~\ref{app:moreexperiments}.  

\subsection{Neural-symbolic attention}\label{subsec:dates}

% It is difficult to compare the performance of \dspl to other, existing methods because of its unique combination of discrete-continuous probabilistic logic and neural networks.
A problem that cannot yet be solved to a satisfactory degree by purely neural or other neural-symbolic systems is detecting handwritten years.
Given a single image with a handwritten year, the task is to predict the correct year as a sequence of 4 digits together with the location of these digits (Figure~\ref{fig:probattention}, left).
This year can be anywhere in the image and the only supervision is in the digits of the year, \emph{not} where these digits are in the image. In other words, the problem is equivalent to object detection \emph{without} bounding box supervision.

Solving such a problem seems to be out of scope for current methods.
On the one hand, existing neural approaches are often complex pipelines of neural components that break end-to-end differentiability~\citep{seker2022generalized}.
On the other hand, current neural-symbolic methods lack sufficient spatial reasoning capabilities in order to perform the necessary image segmentation.

We exploit probabilistic programming by modelling the location of a digit as a deep generalised normal distribution~\citep{nadarajah2005generalized}. 
That is, we use a convolutional neural network to regress the parameters of four generalised normal distributions, one for each digit of a year.
Then, we take inspiration from the spatial transformer literature~\citep{carion2020end} and convert the distribution of each location to an attention map (Figure~\ref{fig:probattention}, right).

\begin{figure}
    \centering
    \includegraphics[width=0.35\linewidth]{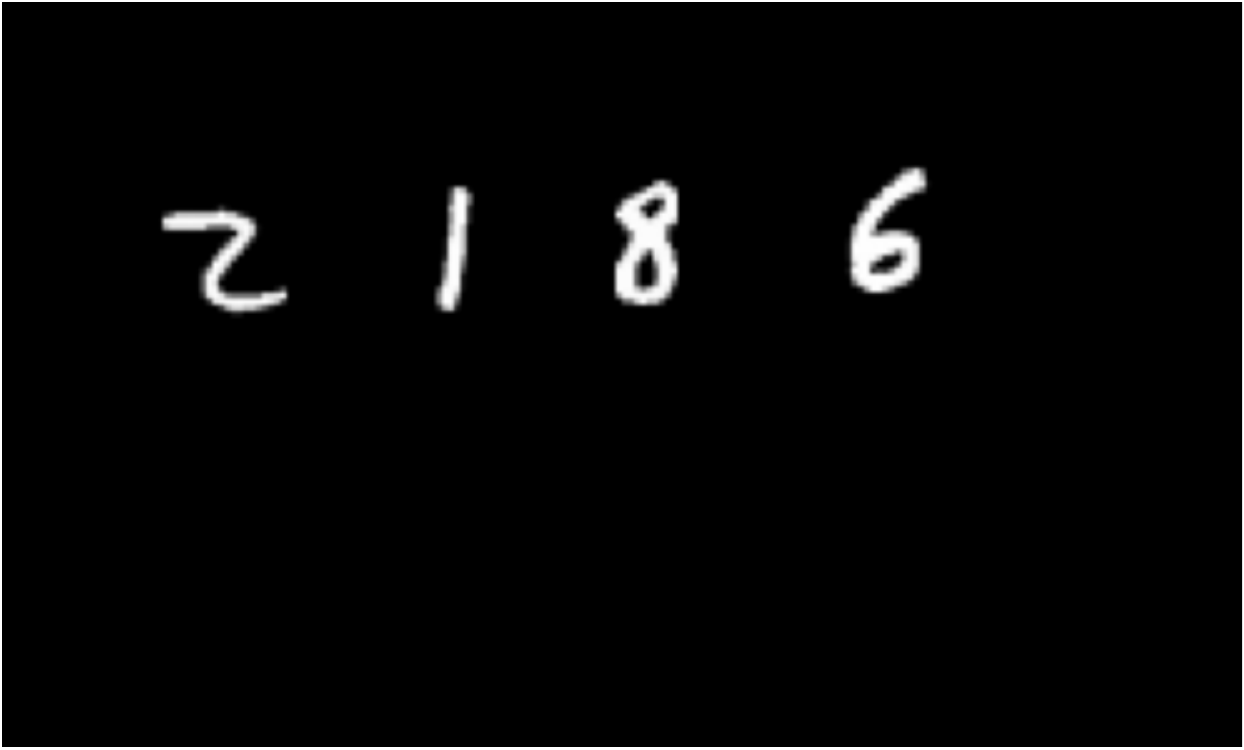}
    \hspace{0.10\linewidth}
    \includegraphics[width=0.35\linewidth]{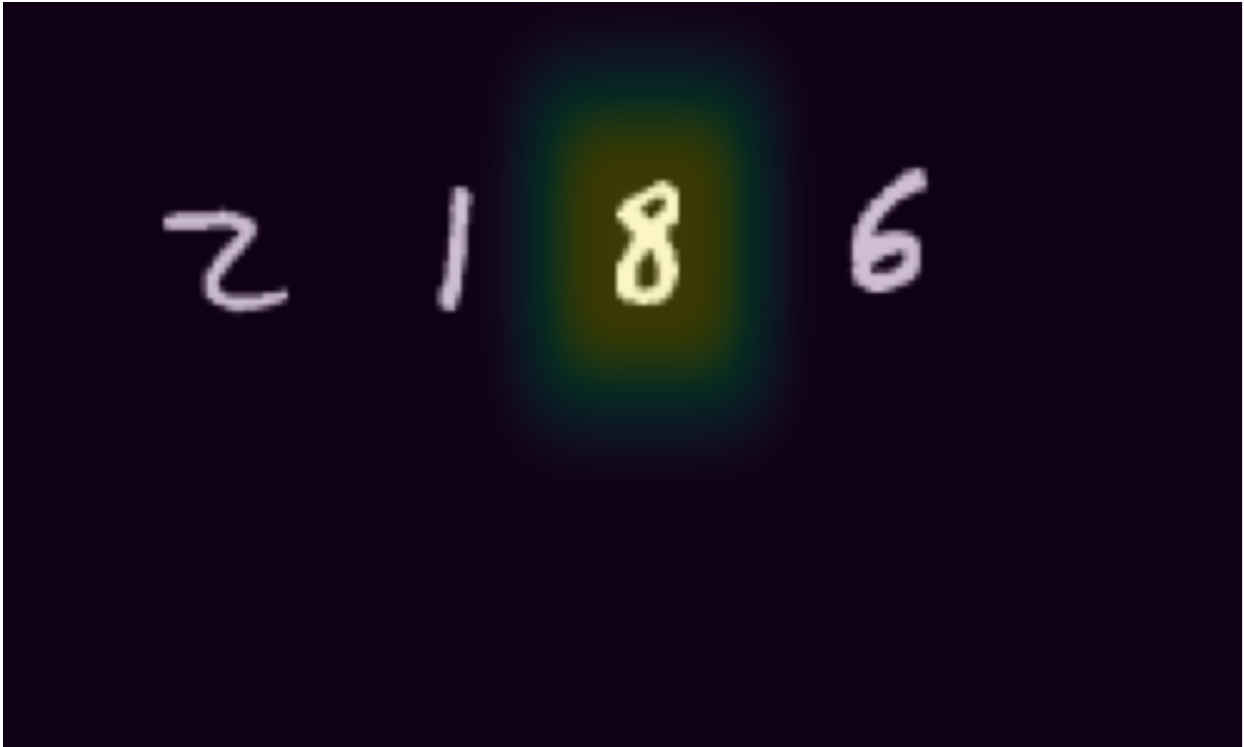}
    \caption{
    On the left, an example of a handwritten year.
    % The only supervision on the left image is the tuple $(2, 1, 8, 6)$. 
    % The attenuated image (right) should get the prediction $8$ to be correct.
    On the right, the attention map for the digit \q{8} as a generalised normal distribution.
    % an illustration of the attention matrix on top of the original image using a generalised normal distribution (right).
    % Notice the soft edges of the \q{bounding box}, which enable differentiation.
    % Such attention facilitates differentiation while retaining the crisp notion of a bounding box.
    Intuitively, we can view generalised normal distributions as differentiable bounding boxes. This allows gradients to flow from a downstream classification network to the regression component.
    }
    \label{fig:probattention}
\end{figure}

In our experimental validation we compare \dspl to a neural baseline and logic tensor networks.
The neural baseline applies the four probabilistic attention maps, one for the location of each of the four digits, to the input image.
The resulting four attenuated images are then passed on to a classification network without additional reasoning.
The network is simply required to predict the digits in the right order.
% The neural baseline applies the four probabilistic attention map to the input image for each of the four digits.
% The resulting four attenuated images are then passed on to a classification network.
With \dspl, we encode that a year is a sequence of digits, i.e., the order matters, by enforcing an explicit order on the digit locations.
Doing so requires spatial reasoning, i.e., reasoning which digit is at which location.
% With \dspl we additionally encode the information that a year is not a set of digits, but a sequence of digits, i.e., the order matters.
For LTNs, we encode the same information.
However, as LTNs lack a proper distribution semantics, they can only reason on the level of the expected values of the generalised normal distributions. 

In our experiment, the sets of years appearing in the training, validation and test data are all disjoint.
Moreover, the sets of handwritten digits used to generate those years are also disjoint.
Partitioning the data in such a way leads to a challenging learning problem; 
the difficulty lies in out-of-distribution inference, as the years and handwritten digits in the validation and test set have never been seen during training.

We evaluate all methods in terms of accuracy and Intersection-over-Union (IoU).
For the accuracy, we compare the sequence of predicted digits to the correct sequence of digits constituting a year.
A prediction is correct if \emph{all} digits are correctly predicted in the right order.
For the IoU, we map each predicted generalised normal distribution to a bounding box by using the mean as the centre and the scale parameter as the width of the box. The IoU is then given by the overlap between this box and the true location of the handwritten digit.
% The accuracy of predicting the correct year from an image.
% Second, we evaluate the continuous reasoning component by computing the Intersection-over-Union (IoU) between the predicted and true bounding box of the digits in the image.
% Do note that it is unreasonable to expect substantial IoU values, as there is no supervision at all on these true boxes.
% It is only necessary to predict distributions that cover the digits in the right order.

\begin{table}[t]
    \caption{Mean accuracy and IoU with standard error for classifying the correct year, taken over 10 runs.
    % Note that one can not expect substantial IoU values, as there is no bounding box supervision.
    % The sub- and superscripts indicate $25\%$ and $75\%$ quantiles, respectively, taken over 10 training runs.
    % The quantiles represent the boundaries between which the middle $50\%$ of observed accuracy values lie. 
    % Note that the results of \expone and \exptwo should not be compared, as they are computed on different test sets. All results are reported in percentages.
    }
    \label{table:yearresults}
    \begin{center}
    \begin{tabular}{ccc}
        \toprule
        Method & \multicolumn{2}{c}{Results} \\
        \midrule
        & acc. & IoU \\
        \cmidrule{2 - 3}
        \dspl & $93.77\pm 0.57\phantom{0}$ & $17.69\pm 0.23$  \\

        LTN &  $76.50\pm 12.10$ & $10.73\pm 1.69$ \\
        
        Neural Baseline &  $54.71\pm 14.33$ & \phantom{0}$6.26\pm 1.77$ \\
        % \midrule
        % Method & \multicolumn{2}{c}{\exptwo} \\
        % \midrule
        % & acc. & IoU \\
        % \cmidrule{2 - 3}
        % \dspl & $93.94_{-0.56}^{+1.05}$ & $17.55_{-0.32}^{+0.65}$  \\

        % LTN &  $95.46_{-0.39}^{+0.16}$ & $13.11_{-0.22}^{+0.51}$ \\
        
        % Neural Baseline &  $0.27_{-0.07}^{+0.09}$ & $54.23_{-1.43}^{+2.88}$ \\
        \bottomrule
    \end{tabular}
    \end{center}
\end{table}

We present our results in Table~\ref{table:yearresults}.
The most striking observation is the poor performance and large variance of the neural baseline. 
It fails to predict the location of the digits in the right order, as can be seen from the lower IoU values.
Since classification depends on the predicted locations, these lower values also explain the lack in accuracy.
We can conclude that the neural baseline struggles to generalise to out-of-distribution data.
While LTNs fare better, the high standard error on the accuracy indicates that their continuous reasoning capabilities are insufficient to guarantee consistent solutions.
\dspl distinguishes itself by a higher and more consistent accuracy.
The reason is also clear; \dspl exploits the entire domain of the distribution of each location. This then leads to a higher IoU value that in turn results in a higher accuracy.

% The most striking observation in our results (Table~\ref{table:yearresults}) is the poor performance of the neural baseline, especially when considering experiment \exptwo. In essence, the neural baseline fails to generalise the learned knowledge. While both NeSy methods are able to generalise, \dspl distinguishes itself by better accuracies. The reason also seems clear; \dspl can exploit the full support of its continuous distributions to reason over the bounding boxes, leading to higher IoU values. Since classification depends on good bounding boxes, the higher IoU can explain the increase in accuracy.

\subsection{Neural hybrid Bayesian networks}\label{subsec:hybridnet}

Hybrid Bayesian networks~\citep{lerner2003hybrid} are probabilistic graphical models that combine discrete and continuous random variables.
\dspl allows for the introduction of optimisable neural components and logical constraints to such models, as shown in Example~\ref{ex:dsplprogram}.
% We further extend this example (Figure~\ref{fig:weathergraph}) such that the parameters of some variables are determined by neural networks from sub-symbolic inputs.
We further extend this example (Figure~\ref{fig:weathergraph}) and specify the datasets that form the input to the various neural networks.
The temperature is predicted from a real meteorological dataset~\citep{cho2020comparative} and we use CIFAR-10 images as proxies for observing clouds and humidity.
% A synthetic dataset was generated with images of 2 objects together with their initial velocity to emulate possible accidents. Predicting the probability of such accidents is done by localising the two objects and applying Newtonian mechanics in \dspl.
Moreover, dependencies on a number of constraints are added, which goes beyond the capabilities of traditional probabilistic programming.
% The exact \dspl program and more details are given in Appendix~\ref{app:moreexperiments}.

\begin{figure}
    \centering
    \includegraphics[width=0.25\linewidth]{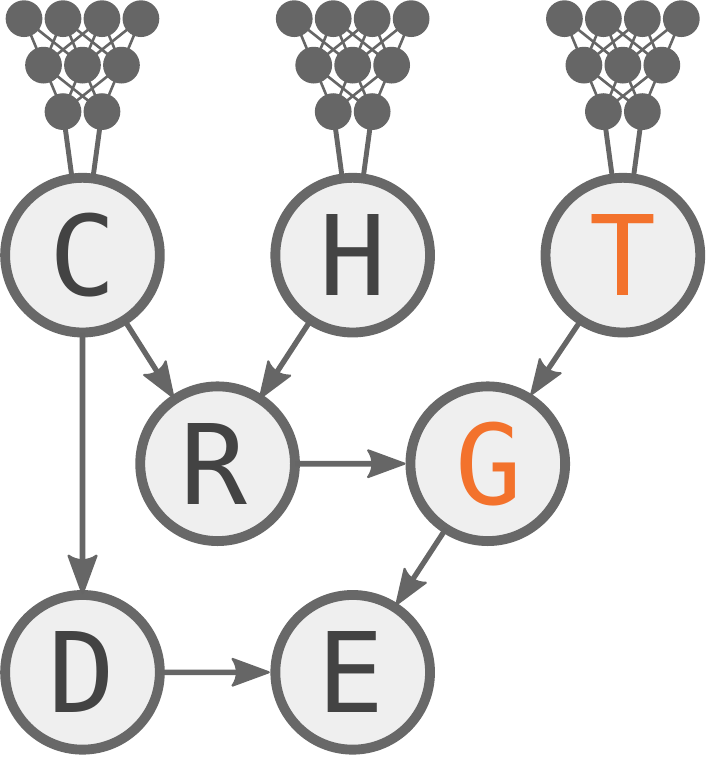}
    \caption{Graphical model of \textbf{E}njoying the weather (\textbf{E}).(\textbf{E}) holds when \textbf{D}epressed (\textbf{D}) is not true and there is \textbf{G}ood weather (\textbf{G}).
    A person has a higher probability of being depressed when it is \textbf{C}loudy (\textbf{C}), while the degree of good weather is beta distributed depending on various logical constraints on \textbf{T}emperature (\textbf{T}) and \textbf{R}ain (\textbf{R}).
    % A person has a higher probability of being depressed when it is \textbf{C}loudy (\textbf{C}) and even more when an \textbf{A}ccident (\textbf{A}) has happened, which depends on the \texbf{L}ocations (\texbf{L}) of 2 objects. The degree of good weather is beta distributed depending on various logical combinations of \textbf{T}emperature (\textbf{T}) ranges and \textbf{R}ain (\textbf{R}).
    Finally, rain is probable when it is both \textbf{C}loudy and \textbf{H}umid (\textbf{H}).}
    \label{fig:weathergraph}
\end{figure}

Our neural Bayesian model was optimised by only giving probabilistic supervision on whether \textbf{E} was true or false, i.e., the weather was enjoyed or not.
Given our model, such distant supervision only translates into a learning signal on different \emph{ranges} of temperature values that satisfy different PCFs. We will see that \dspl's reasoning over the full domain of the temperature distribution allows it to perform meaningful density estimation from such a signal.

The optimised Bayesian model can be evaluated in two ways. First, the accuracy on CIFAR-10 of the networks utilised in \textbf{C}loudy and \textbf{H}umid, which were
% $99.08_{-1.38}^{+0.13}$ and $99.10_{-0.10}^{+0.00}$
$95.24\pm 3.32$ and $98.96\pm 0.11$, respectively.
Second, we measure the quality of the density estimation on \textbf{T}emperature by looking at the MSE between the true and predicted mean values, which was 
% $0.158_{-0.013}^{+0.0385}$
$0.1799 \pm 0.0139$
.
% Second, the MSE between the true and predicted mean values for \probloginline{temperature}, being $TBD$. 
Importantly, \dspl was able to approximate the standard deviation of \textbf{T}emperature from just the distant supervision, deviating by only
% $0.33_{-0.21}^{+0.43}$
$0.60\pm 0.22$. 
% Third, the accuracy on predicting whether an accident is likely or not, being $TBD$. 

% Note how this neural and hybrid Bayesian model is a prototypical example of how to exploit the complex dependencies of a probabilistic world. It can be seen as the first step towards being able to successfully apply neural-symbolic principles to domains such as robotics, which involves logical reasoning over discrete and continuous random variables. While \dspl currently lacks the capacity to deal with the dynamics of such an application, it is a promising language in which to symbolically model the uncertain environment of an autonomous agent.

\subsection{Neural-symbolic variational auto-encoder}\label{subsec:vae}

% In the previous experiment we predicted the difference between two digits by regarding it as a classification problem.
% Now we would like to run this experiment in reverse by generating two MNIST digits given a subtraction result.
Probabilistic programming is well-suited to generative tasks, but it can not perform generation conditioned on logical constraints.
Inspired by the work of~\citet{misino2022vael}, we showcase how \dspl extends the generative power of probabilistic programming to such constraints.
To this end, we will consider the task of learning to generate 2 images of digits given the value of their subtraction.
% we opt for a conditional variational auto-encoder approach  (CVAE)~\citep{kingma2015adam, sohn2015learning}.
% A diagrammatic overview is given in Figure~\ref{fig:vae_diagram}.

A diagrammatic overview of our \dspl program is given in Figure~\ref{fig:vae_diagram}.
It uses a conditional variational auto-encoder (CVAE)~\citep{sohn2015learning} to generate images conditioned on a digit value.
\dspl finds those digit values from a given subtraction result by logical reasoning.
It can also condition generation on other variables in the CVAE latent space as this space is an integral part of \dspl's deep, relational model.
% It can also impose and exploit structure on the CVAE latent space since this space is an integral part of \dspl's deep, relational model.
We will exploit this property later on when we extend the task to generating digits in the same writing style as a given image without \emph{any} additional optimisation.
% More details on our implementation can be found in Appendix~\ref{app:morevae}.
% Note that other NeSy frameworks, such as LTNs, can not express such deep, relational and generative models as they lack the necessary probabilistic semantics.
% The flexibility of \dspl allows to declaratively encode the architecture presented in Figure~\ref{fig:vae_diagram} as a neural probabilistic logic program.

\begin{figure}[ht]
    \centering
    \includegraphics[width=0.7\linewidth]{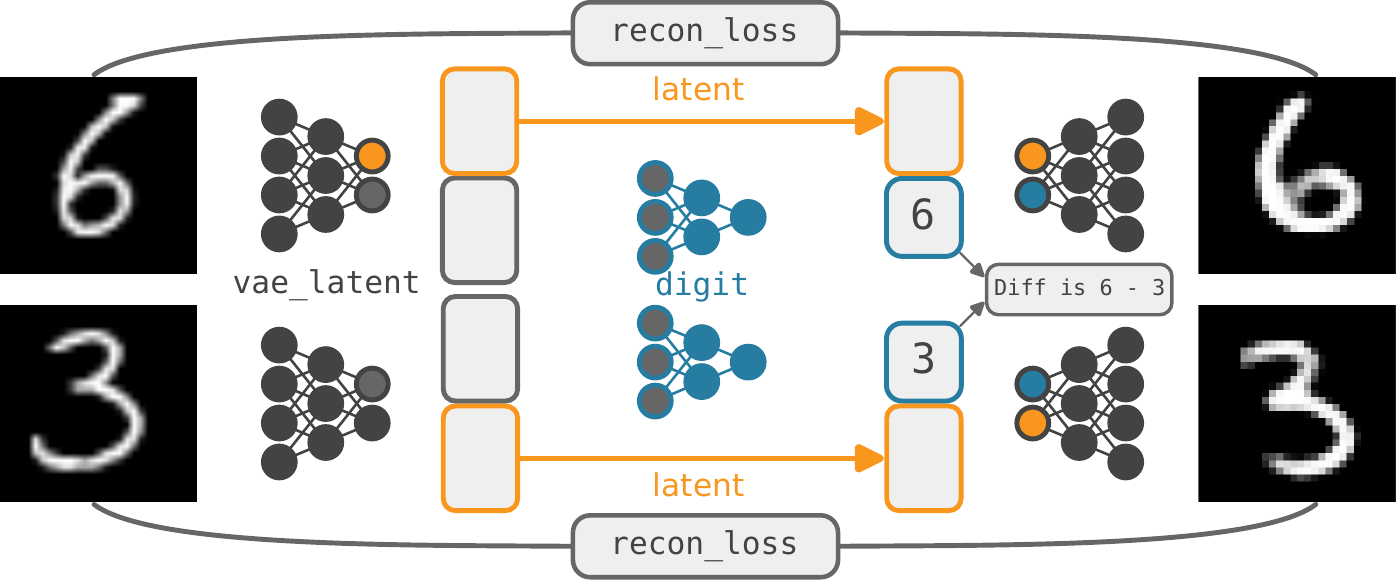}
    \caption{
    % Each image is encoded into a multivariate normal NDF (\textcolor{orange_red}{\probloginline{shape}}) and a latent vector by \probloginline{vae_latent}. The latter forms the input to the NDF \textcolor{celadon_blue}{\probloginline{digit}}, while a sample of the former is combined with the 
    % outcome of \probloginline{digit} to form the input of the decoder.
    % Doing so for both images yields two reconstructions, which are compared to the original images in a probabilistic \probloginline{recon_loss}. Note that the values of \textcolor{celadon_blue}{\probloginline{digit}} for both images also have to comply to the value of the given difference.
    Given example pairs of images and the value of their subtraction, e.g., $(\inlineimage{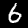}, \inlineimage{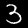})$ and $3$, the CVAE encoder \probloginline{vae_latent} first encodes each image into a multivariate normal NDF (\textcolor{orange_red}{\probloginline{latent}}) and a latent vector. The latter is the input of a categorical NDF \textcolor{celadon_blue}{\probloginline{digit}}, completing the CVAE latent space. Supervision is dual; generated images are compared to the original ones in a probabilistic reconstruction loss, while both digits need to subtract to the given value.
    }
    \label{fig:vae_diagram}
\end{figure}
% using the input images themselves as supervision in a reconstruction loss. Additionally, the result of the difference between the digits has to comply with a given value. For instance, if images
% \inlineimage{Imagery/2_6_[6]_[7]_original.png}
% and
% \inlineimage{Imagery/5_3_[29]_[2]_original.png}
% are given, then we give the label $3$ as additional supervision.

Both the CVAE and digit classifier are successfully trained jointly. Example generations of images that satisfy the subtraction result
$
\inlineimage{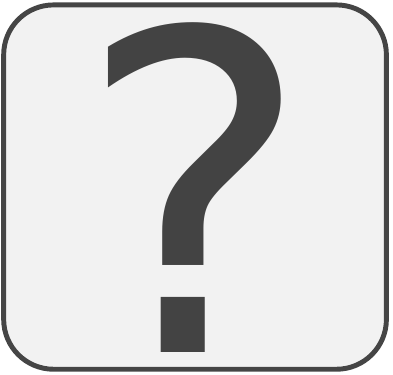}
{-}
\inlineimage{Imagery/generation.pdf}
= 5
$
can be seen below. In general, \dspl finds all possible digits that subtract to a given value and generates images for each correct combination. Below, we left out 2 such combinations for clarity of exposition.
% After this training phase, pairs of digits that result in a specific subtraction result can now be generated. To do this, we first sample from the normally distributed latent space of the NeSy-VAE, which generates two latent representations for two digits. Next, the logic deduces which digits comply with the given subtraction result and attaches these to the two latent representations. Finally, these representations are passed through the learned decoder, which constructs two images that satisfy the subtraction result 
% $
% \inlineimage{Imagery/generation.pdf}
% {-}
% \inlineimage{Imagery/generation.pdf}
% = 5
% $
% (below).

\begin{figure}[ht]
\includegraphics[width=0.12\linewidth]{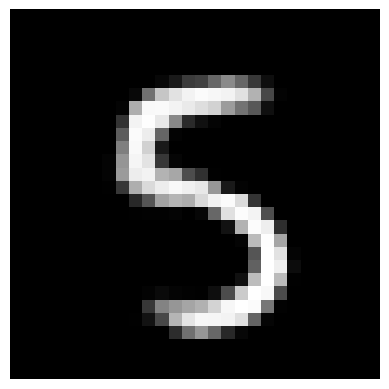}
\includegraphics[width=0.12\linewidth]{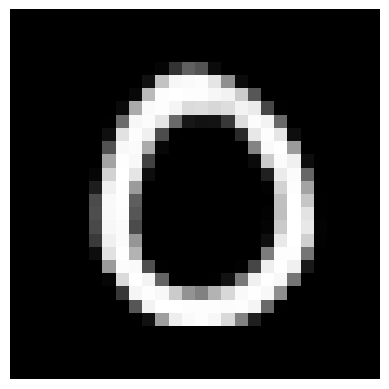}
\hfill
\includegraphics[width=0.12\linewidth]{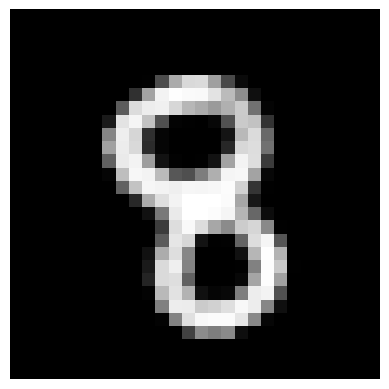} 
\includegraphics[width=0.12\linewidth]{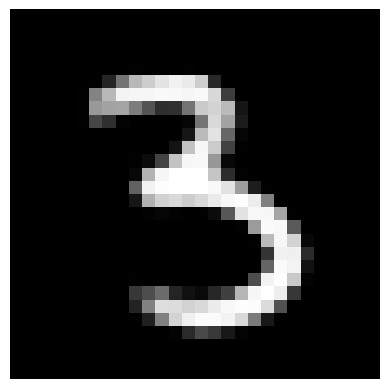}
\hfill
\includegraphics[width=0.12\linewidth]{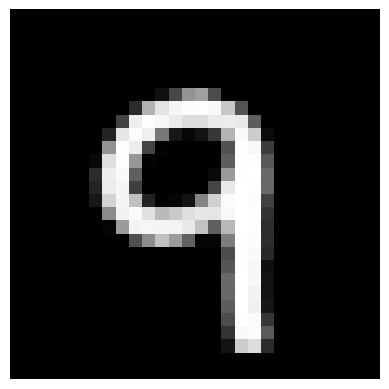} 
\includegraphics[width=0.12\linewidth]{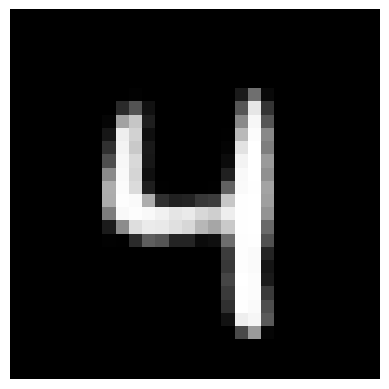}
\end{figure}

While our program is inspired by the VAEL architecture of \citet{misino2022vael}, conceptual differences exist.
Most notably, for VAEL, the image generation resides outside the probabilistic logic program. Conversely, the CVAE, including its latent space, is explicitly declared and accessible in \dspl.
This difference allows \dspl to generalise to conditional generative queries that differ significantly from the original optimisation task.
For example, we can \emph{zero-shot} query the program to fill in the blank in 
% For example, we \emph{zero-shot} query the program to generate an image of the right digit in a subtraction expression given an image of the left one and the subtraction result.
% That is, we ask the program to fill in the blank in
% More precisely, without performing any retraining, we query the \dspl program to generate a subtrahend given a minuend and the subtraction result, i.e., fill in the blank in
$
\inlineimage{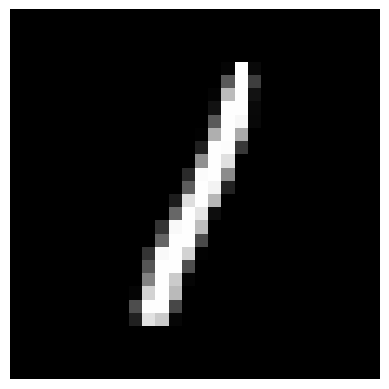}
{-}
\inlineimage{Imagery/generation.pdf}
= \text{\probloginline{Diff}}$
instead of the two blanks of the learning task
$
\inlineimage{Imagery/generation.pdf}
{-}
\inlineimage{Imagery/generation.pdf}
= \text{\probloginline{Diff}}$.
Even more, we can enforce that the generated digit is in the same writing style as the given digit by conditioning the generation on the latent space of the given image (Figure~\ref{fig:querygens}).
% For the leftmost example in Fi left to right, generated subtrahend images for a \text{\probloginline{Diff}} value of 0, -5 and -7 respectively.
\begin{figure}[ht]
% \begin{minipage}{0.16\linewidth}
% \includegraphics[width=\linewidth]{Imagery/imageconditionalgeneration_input2_grey.png}
% \vspace{5pt}
% \centering
% \includegraphics[width=0.25\linewidth]{Imagery/conditionalgen2_0diff.png}
% \includegraphics[width=0.25\linewidth]{Imagery/conditionalgen2_minus5diff.png}
% \includegraphics[width=0.25\linewidth]{Imagery/conditionalgen2_minus7diff.png}
% \end{minipage}
% \hfill
\begin{minipage}{0.16\linewidth}
\includegraphics[width=\linewidth]{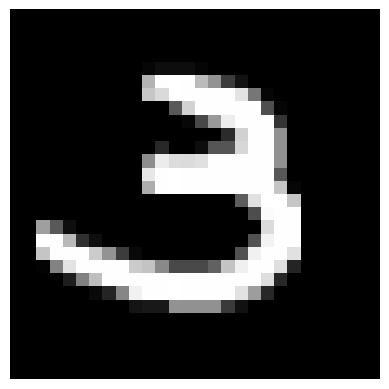}
\vspace{5pt}
\centering
\includegraphics[width=0.30\linewidth]{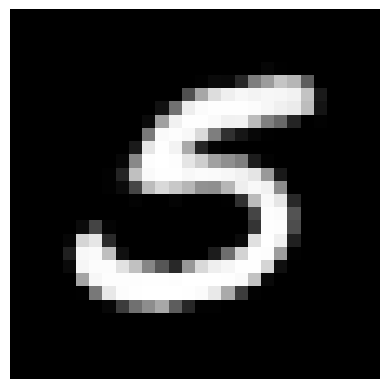}
\includegraphics[width=0.30\linewidth]{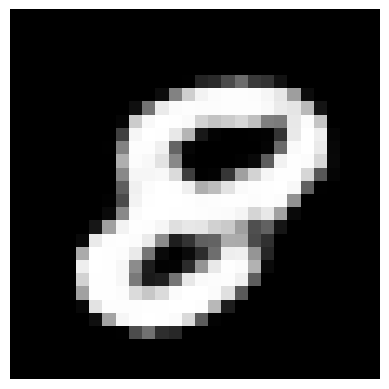}
\includegraphics[width=0.30\linewidth]{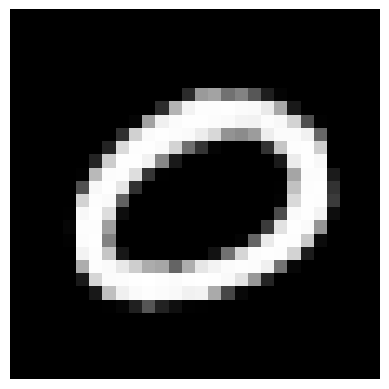}
\end{minipage}
\hfill
\begin{minipage}{0.16\linewidth}
\includegraphics[width=\linewidth]{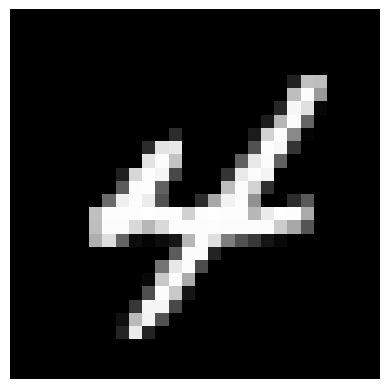}
\vspace{5pt}
\centering
\includegraphics[width=0.30\linewidth]{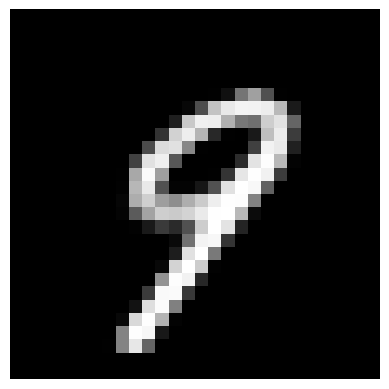}
\includegraphics[width=0.30\linewidth]{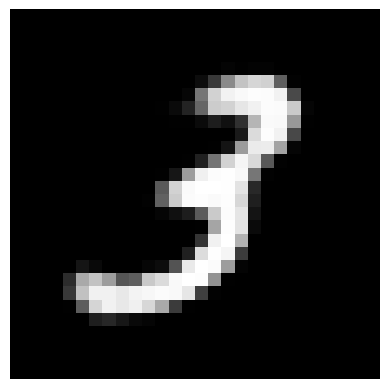}
\includegraphics[width=0.30\linewidth]{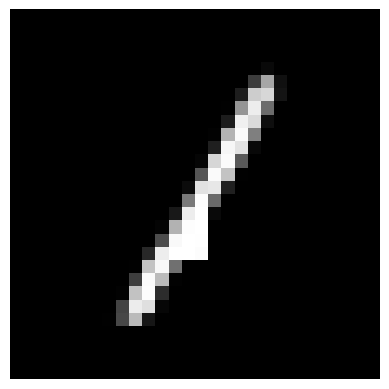}
\end{minipage}
\hfill
\begin{minipage}{0.16\linewidth}
\includegraphics[width=\linewidth]{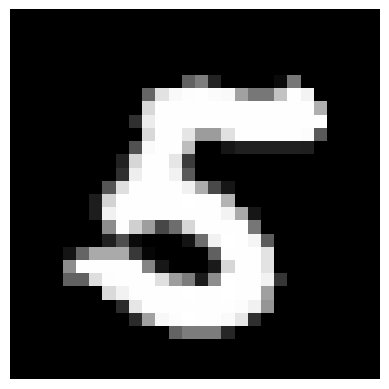}
\vspace{5pt}
\centering
\includegraphics[width=0.30\linewidth]{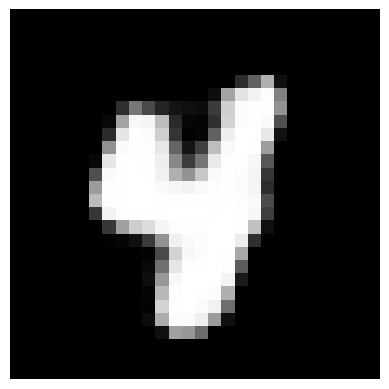}
\includegraphics[width=0.30\linewidth]{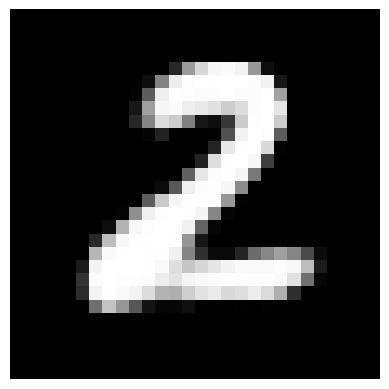}
\includegraphics[width=0.30\linewidth]{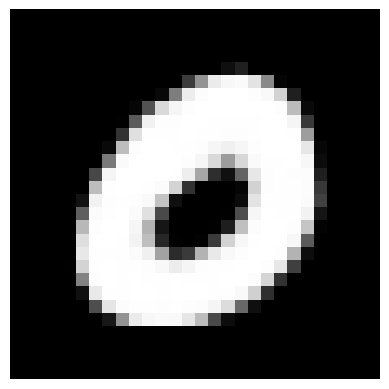}
\end{minipage}
\hfill
\begin{minipage}{0.16\linewidth}
\includegraphics[width=\linewidth]{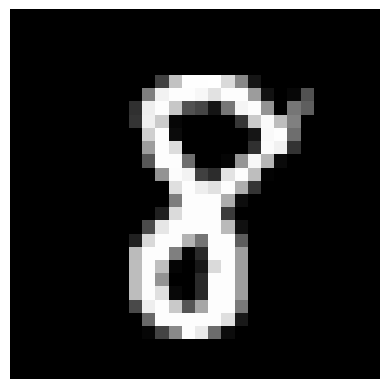}
\vspace{5pt}
\centering
\includegraphics[width=0.30\linewidth]{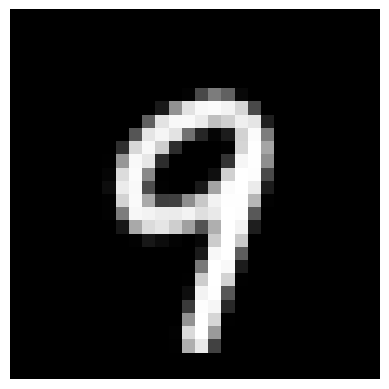}
\includegraphics[width=0.30\linewidth]{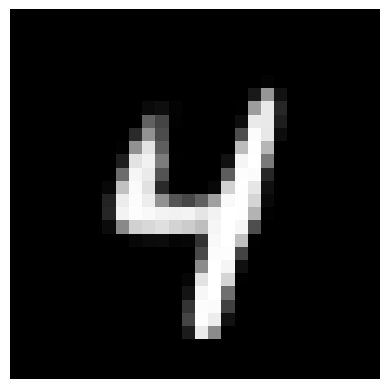}
\includegraphics[width=0.30\linewidth]{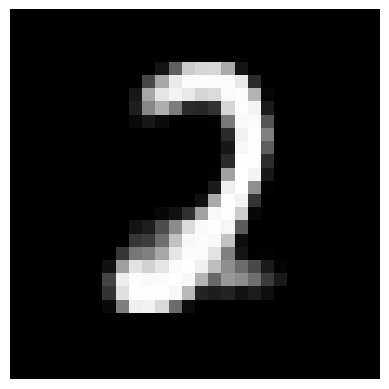}
\end{minipage}
\centering
\caption{Four random images of right digits (top row) and their generated left digits for 3 given random difference values (bottom row). Note the preservation of the style of the given minuends.}
\label{fig:querygens}
\end{figure}
% \end{minipage}
% \hfill
% \begin{minipage}{0.30\linewidth}
% \hfill
% \includegraphics[width=\linewidth]{Imagery/imageconditionalgeneration_input2_grey.png}
% \vspace{5pt}
% \includegraphics[width=0.30\linewidth]{Imagery/conditionalgen2_0diff.png}
% \includegraphics[width=0.30\linewidth]{Imagery/conditionalgen2_minus5diff.png}
% \includegraphics[width=0.30\linewidth]{Imagery/conditionalgen2_minus7diff.png}
% \end{minipage}

% Having demonstrated its generative and reasoning capabilities, we conclude that \dspl bridges the DPP-NeSy gap and affirmatively answer \qgap.

\section{Conclusion}\label{sec:conclusion}

We presented \dspl, a novel neural-symbolic probabilistic logic programming language that integrates hybrid probabilistic logic and neural networks. Inference is dealt with efficiently through approximate weighted model integration while learning is facilitated by reparametrisation and continuous relaxations of non-differentiable logic components. Our experiments illustrate how \dspl is capable of intricate probabilistic modelling allowing for meaningful weak supervision while maintaining strong out-of-distribution performance. Moreover, they show how hybrid probabilistic logic can be used as a flexible structuring formalism for the neural paradigm that can effectively optimise and reuse neural components in different tasks. 

\newpage

\bibliography{references}  

\clearpage

\appendix
\renewcommand\thefigure{\thesection.\arabic{figure}}    

\section{Special cases of \dspl}
\label{app:special}

The syntax and semantics of \dspl generalise a number of probabilistic logic programming dialects. For instance, if we assume no dependency of the distributional facts on input data or external neural functions, we obtain a language equivalent to~\citeauthor{gutmann2011magic}'s  {\em Distributional Clauses} (DC)~\citep{gutmann2011magic} when restricted to distributional facts. Finally, if we allow for data dependent neural functions in the NDFs but restrict them to Bernoulli and categorical distributions, we obtain~\citeauthor{manhaeve2021neural}'s \dpl~\citep{manhaeve2021neural} as a special case.
\begin{restatable}[\dspl strictly generalises \dpl]{proposition}{deepproblogspecial}
\label{proposition:deepproblogspecial}
\dpl is a strict subset of \dspl where the set of comparison predicates is restricted to $\{ \text{\probloginline{=:=}} \}$, comparisons involve exactly one random variable and the measure $\differential P_{\dfacts}$ factorizes as a product of independent Bernoulli measures
$
    \prod_{i: \text{\probloginline{x@$_i$@~b@$_i$@}}\in \dfacts}
    \differential P_{b_i}
$.
The subscript on $\differential P_{b_i}$ explicitly identifies the measure as the $i^{\text{th}}$ Bernoulli measure and the indices of the product go over all the (Bernoulli) random variables defined in the set of distributional facts $\dfacts$.
\end{restatable}

\begin{prf}
We prove Proposition~\ref{proposition:deepproblogspecial} by showing that applying the restrictions on the constraints and measure in a \dspl program leads to possible worlds that have the same probability of being true as in \dpl.
First we write down the definition of the probability $P(\world_{\compsubset})$ of a possible world in a \dspl program
\begin{align}
    \int
    \left[
        \left(\prod_{c_i\in \compsubset} \indicator(c_i) \right)
        \left( \prod_{c_i \in \compset\setminus \compsubset}  \indicator(\bar{c}_i) \right)
    \right]
    \ \differential P_{\dfacts}.
    \label{eq:proof_deepproblogspecial_1}
\end{align}
Now observe that, since there are only Bernoulli distributions, we only need to consider two possible outcomes of a random variable \probloginline{x@$_i$@}, either zero or one. Therefore, only two kinds of comparisons are present in the program, \probloginline{x@$_i$@=:=0} or \probloginline{x@$_i$@=:=1} (remember that we restrict ourselves to univariate comparisons). Now note that the following equivalence
$
\text{\probloginline{x@$_i$@=:=1}}\leftrightarrow \neg (\text{\probloginline{x@$_i$@=:=0}})
$
holds, which means that we can arbitrarily limit comparisons to one of the two possible outcomes of a random variable, e.g., \probloginline{x@$_i$@=:=0}.

This equivalence can be used to replace the constraints $c_i$ in Equation~\ref{eq:proof_deepproblogspecial_1} by equality constraints involving comparisons to the zero outcome, i.e., $P(\world_{\compsubset})$ is equal to
\begin{align}
\label{eq:integralofproduct}
    \int
    % \left[
        &\left(\prod_{i:c_i\in \compsubset} \indicator(x_i{=}0) \right)
        \cdot \\
        &\left( \prod_{i:c_i \in \compset\setminus \compsubset}  \indicator(x_i{\neq}0) \right)
    % \right]
    \prod_{i: \text{\probloginline{x@$_i$@~b@$_i$@}}\in \dfacts}
    \ \differential P_{b_i},\nonumber
\end{align}    
where the factorisation of the measure was also applied. Next, we introduce the following notation for the random variables present in the set of constraints $\compsubset$ and $\compset\setminus \compsubset$:
\begin{align}
    \variables{x}^+ &\coloneqq {x_i: c_i 
    \in \compsubset}
    \\
    \variables{x}^- &\coloneqq {x_i: c_i 
    \in \compset\setminus \compsubset}
\end{align}
Note that we only need to consider the case where $\variables{x}^+ \cap \variables{x}^-=\emptyset$, as otherwise the probability of the possible world would simply be zero and would not contribute to the overall probability of the query atom. Because of this, we can further factorize the measure as

\begin{align}
    &\prod_{i: \text{\probloginline{x@$_i$@~b@$_i$@}}\in \dfacts}
    \ \differential P_{b_i}
    \\
    &=
    \underbrace{
        \left(
            \prod_{i:x_i\in \variables{x}^+}
            \ \differential P_{b_i}
        \right)
    }_{
        \eqqcolon \differential P^+
    }
    \underbrace{
        \left(
            \prod_{i:x_i \in \variables{x}^-} \ \differential P_{b_i}
        \right)
    }_{
        \eqqcolon \differential P^-
    },\nonumber
\end{align}
so the integral of a product in Equation \ref{eq:integralofproduct} can be rewritten as the product of integrals
\begin{align}
    &P(\world_{\compsubset})
    = \\
    &\left[
    \int 
    \left(
    \prod_{i:c_i\in \compsubset}
        \indicator(x_i{=}0)
        \ \differential P^+
    \right)
    \right]
    \cdot 
    \nonumber
    \\
    &\left[
    \int
    \left(
    \prod_{i:c_i \in \compset\setminus \compsubset}
        \indicator(x_i{\neq}0)
        \ \differential P^-
    \right)
    \right].
    \nonumber
\end{align}
We have two integrals with integrands that are a product of univariate comparisons. In other words, the factors are all independent. Furthermore, we have a Bernoulli product measure, which means that we can again push the integral inside the product to yield
\begin{align}
    &P(\world_{\compsubset})
    = \\
    &\left[
    \prod_{i:c_i\in \compsubset}
    \left(
        \int
        \indicator(x_i{=}0)
        \ \differential P^+
    \right)
    \right]
    \cdot 
    \nonumber
    \\
    &\left[
    \prod_{i:c_i \in \compset\setminus \compsubset}
    \left(
        \int
        \indicator(x_i{\neq}0)
        \ \differential P^-
    \right)
    \right].\nonumber
\end{align}
At this point we can simply perform the integrations and obtain
\begin{align}
    P(\world_{\compsubset})
    &=
    \prod_{i:c_i\in \compsubset}
    p_{i}
    \prod_{i:c_i \in \compset\setminus \compsubset}
    (1-p_{i}),
\end{align}
which coincides with the probability of a possible world in \dpl \citep[Section 3]{manhaeve2021neural}.
\end{prf}

Proposition~\ref{proposition:deepproblogspecial} can easily be extended to also allow for measures of finite categorical distributions, which then translates to (neural) annotated disjunctions. 
Consequently, as \dpl is a strict superset of ProbLog~\citep{fierens2015inference}, \dspl also strictly generalises ProbLog.

\section{Proof of Proposition~\ref{proposition:query_measurability}}
\label{app:proof_query_probability}

\querymeasurability*
\begin{prf}
\dspl is in essence a subset of the probabilistic logic programming language defined by~\citet{gutmann2011magic} -- the only difference being that the parameters on the right-hand side of a neural distributional fact are not limited to numerical constants any more but can be arbitrary numeric terms.
Under the condition that all NDFs and PCFs are valid, this does, however, not violate any of the assumptions made in~\cite[Proposition 1]{gutmann2011magic} (proving the measurability of a program).
We can, hence, conclude that a valid \dspl program induces a probability measure for $q$.
\end{prf}

Note that, similar to ProbLog and \dpl, the semantics for \dspl are only defined for so-called sound programs~\citep{riguzzi2013well}, which means that all programs become ground eventually when queried.

\section{Proof of Proposition~\ref{proposition:inferenceaswmi}}
\label{app:proof_inference_as_wmi}

{\renewcommand\footnote[1]{}\inferenceaswmi*}
\begin{prf}
First, let us consider the indices of the two product expressions in Equation~\ref{eq:world_probability}. We define
\begin{align*}
    \negcompsubset \coloneqq \left\{\bar{c}_i\ |\ c_i \in \compset {\setminus} \compsubset\right\}
\end{align*}
such that Equation~\ref{eq:world_probability} can be rewritten as
\begin{align}
    P(\world_{\compsubset}) =
        \int
            \left(\prod_{c_i \in \compsubset {\cup} \negcompsubset}  \indicator(c_i(\variables{x}) \right)
        \ \differential P_{\dfacts}
\end{align}
Furthermore, decomposing the measure into a probability density function $\weight(\variables{x})$ and a differential $\differential\variables{x}$ of the integration variables yields
\begin{align}
        \int
            \left(\prod_{c_i\in \compsubset {\cup} \negcompsubset}  \indicator(c_i(\variables{x})) \right)
            \cdot \weight(\variables{x})
        \ \differential \variables{x}.
\end{align}
We can now plug this last expression into Equation~\ref{eq:query_probability} resulting in
\begin{align}
    P(q) =
    \int
        \sum_{\substack{C_M \subseteq \mathcal{C}_M:\\ q \in \omega_{\mathcal{C}_M}}}
        \left(
        \prod_{c_i\in \compsubset {\cup} \negcompsubset}  \indicator(c_i(\variables{x}))
        \right)
    \cdot \weight(\variables{x})
    \ \differential \variables{x}.
    \label{eq:proof_inference_as_wmi}
\end{align}
Note that we changed the order of the integration and summation. This operation was shown to be valid in~\citet{zuidberg2019exact} using de Finetti's theorem. \citet{zuidberg2019exact} also showed that the expression in Equation~\ref{eq:proof_inference_as_wmi} is indeed a weighted model integral as defined by~\citet{belle2015probabilistic}. Specifically, line P2 in the proof of Theorem 2 in \citet{zuidberg2019exact} corresponds to C.3, which is shown to be equal to an instance of WMI.
\end{prf}

\section{Details on derivative estimate}
\label{app:detaileddiff}

To give further details on estimating the derivative we will write the expression $\deriv P_{\Neuralparams}(q)$ in terms of indicator functions
\begin{align}
    &\deriv P_{\Neuralparams}(q)
    \\
    &=
    \deriv \int \amc(\variables{x}) \cdot \weight_{\Neuralparams}(\variables{x})\ \partial \variables{x} \nonumber
    \\
    &= \nonumber
    \deriv \int 
        \sum_{\substack{C_M \subseteq \mathcal{C}_M:\\ q \in \omega_{\mathcal{C}_M}}} 
        \left(
        \prod_{c_i\in \compsubset {\cup} \negcompsubset}  \indicator(c_i(\variables{x}))
        \right)
    \cdot \weight_{\Neuralparams}(\variables{x})
    \ \differential \variables{x}
    ,
\end{align}
where the dependency of the probability on the neural parameters $\Neuralparams$ is again made explicit.
Reparametrising the distribution $\weight_{\Neuralparams}(\variables{x})$ yields
\begin{align}
\label{eq:detailedreparamder}
    &\deriv P_{\Neuralparams}(q)
    = \\
    &\deriv\int
    \sum_{\substack{C_M \subseteq \mathcal{C}_M:\\ q \in \omega_{\mathcal{C}_M}}} 
        \left(
            \prod_{c_i\in \compsubset {\cup} \negcompsubset}  \indicator(c_i(\reparam)
        \right)
    \cdot p(\variables{u}) 
    \ \mathrm{d}\boldsymbol{u}. \nonumber
\end{align}
Explicitly writing out the indicators clearly illustrates the non-differentiability of $\amc(\variables{x})$, which prevents us from applying Leibniz' integral rule \citep{flanders1973differentiation} to swap the order of integration and differentiation. To obtain the necessary differentiability of the integrand, the continuous relaxations introduced by~\citet{petersen2021learning} are utilised. These relaxations allow for comparison formulae of the form
\begin{align}
    (g(\boldsymbol{x}) \bowtie 0),
    \quad \text{with}
    \bowtie\ \in\ \left\{<, \leq, >, \geq, =, \neq\right\}
\end{align}
to be relaxed. We write the continuous relaxation of an indicator function 
$\indicator(c_i(\variables{x}))
=
\indicator(g_i(\variables{x})\bowtie 0)$ as
$
    \relaxation_i(\variables{x})
$.
Four specific cases of relaxations arise, depending on the comparison operator used. Specifically, we define

\begin{align}\label{eq:adaptedsoft}
    &\relaxation_i(\variables{x}) =
    \\ \nonumber
    &\begin{cases}
    \sigma(\coolness_i \cdot g_i(\variables{x}))
    &
    \text{if $\bowtie\ \in\ \left\{>, \geq \right\}$}, \\
    \sigma(-\coolness_i \cdot g_i(\variables{x}))
    &
    \text{if $\bowtie\ \in\ \left\{ <, \leq\right\}$}, \\
    \prod_{}
    \sigma(\coolness_i \cdot g_i(\variables{x})) \cdot \sigma(-\coolness'_i \cdot g_i(\variables{x}))
    &
    \text{if $\bowtie\ \in\ \left\{=\right\}$},
    \\
    1 - \sigma(\coolness_i \cdot g_i(\variables{x})) \cdot \sigma(-\coolness'_i \cdot g_i(\variables{x}))
    &
    \text{if $\bowtie\ \in\ \left\{\neq\right\}$},
    \end{cases}
\end{align}
where $\coolness_i$ and $\coolness'_i$  are the coolness parameters of the continuous relaxations and $\sigma$ denotes the sigmoid function.
Note that all four cases originate from the root choice of approximating the step function as a sigmoid function. Additionally, this choice is sound as we have that 
\begin{equation}
    \lim_{\coolness_i \rightarrow +\infty} \sigma(\coolness_i \cdot g_i(\variables{x})) = \indicator(g_i(\variables{x}) \geq 0).
\end{equation}

Continuously relaxing indicator functions using the definition of Equation~\ref{eq:adaptedsoft} renders the integrand differentiable, allowing the application of Leibniz' integral rule and yielding
\begin{align}
    &\deriv P_{\Neuralparams}(q)
    \approx \\
    \nonumber
    &\int  \deriv
    \sum_{
            \substack{\compsubset \subseteq \compset:
            \\
            q \in \omega_{\compset}}
        }
        \left(
        \prod_{i: c_i\in \compsubset {\cup} \negcompsubset}  \relaxation_i(\reparam)
        \right)
    \cdot p(\variables{u})
    \ \differential \variables{u}.
\end{align}

The derivative $\deriv P_{\Neuralparams}(q)$ can now be computed using off-the-shelf automatic differentiation software such as PyTorch \citep{paszke2019pytorch} or TensorFlow \citep{abadi2016tensorflow}, which entails that estimating the gradient  $\nabla_{\Neuralparams} P(q) = (\deriv P(q))_{\lambda \in \variables{\Neuralparams}}$ is computationally as expensive as computing the probability itself, up to a constant factor~\citep{griewank2008evaluating}.

\section{Proof of proposition~\ref{proposition:unbiasedapprox}}
\label{app:unbiasedness}

\unbiasedapprox*

\phantom{=}
\begin{prf}
First we express $P(q)$ using Equation~\ref{eq:proof_inference_as_wmi}, which we then rewrite without loss of generalisation using only Heaviside distributions\footnote{Here we use the term \emph{distribution} in the sense of a generalised function~\citep{schwartz1957theorie} and not in the sense of a probability distribution.}.
\begin{align}\label{eq:prob_query_heavi}
    &P(q) \\ 
    &= \nonumber
    \int
        \sum_{\substack{C_M \subseteq \mathcal{C}_M:\\ q \in \omega_{\mathcal{C}_M}}}
        \left(
        \prod_{c_i\in \compsubset {\cup} \negcompsubset}  \indicator(c_i(\variables{x}))
        \right)
    \cdot \weight(\variables{x})
    \ \differential \variables{x} \\
    &= \nonumber
    \int
    \sum_{\substack{C_M \subseteq \mathcal{C}_M:\\ q \in \omega_{\mathcal{C}_M}}} 
            \left(
                % \prod_{c_i \in \compsubset {\cup} \negcompsubset}
                \prod_{g_i \in \Sigma_{\compsubset {\cup} \negcompsubset}}
                H(g_i(\reparam))
            \right)
    \cdot p(\variables{u}) 
    \ \mathrm{d}\boldsymbol{u}
    .
\end{align}
In the Equation above, $H(x)$ denotes the Heaviside distribution and $\Sigma_{\compsubset {\cup} \negcompsubset}$ is the set of all sigmoid functions involved in the continuous relaxations of the set $\compsubset {\cup} \negcompsubset$.

This rewrite is possible as the indicator function of any PCF $c(\variables{x})$ is either a step function or decomposes into a product of step functions. Indeed, if $c(\variables{x})$ is of the form $g(\boldsymbol{x}) \geq 0$, then $\indicator(c(\variables{x})) = H(g(\boldsymbol{x}))$. If it is of the form $g(\boldsymbol{x}) = 0$, then $\indicator(c(\variables{x})) = H(g(\boldsymbol{x})) \cdot H(- g(\boldsymbol{x}))$. The other cases with different comparison operators follow from these two.

Differentiating in a distributional sense and applying Leibniz' integral rule~\citep{flanders1973differentiation} then yields
\begin{align}\label{eq:heavisideder}
    \sum_{\substack{C_M \subseteq \mathcal{C}_M:\\ q \in \omega_{\mathcal{C}_M}}}
    \sum_{g_j \in \Sigma_{\compsubset {\cup} \negcompsubset}}
    \int
    &\deriv H(g_j(\reparam)) \cdot \\ \nonumber
    &\prod_{i \neq j} H(g_i(\reparam))
    \cdot p(\variables{u}) 
    \ \mathrm{d}\boldsymbol{u}.
\end{align}

We can reduce the discussion by considering each term in this equation separately, because of the linearity of the integral. In other words, to prove our statement, it suffices to show that
\begin{align}\label{eq:heavisideterm}
    \int
    &\deriv H(g_j(\reparam)) \cdot \\ \nonumber
    &\prod_{i \neq j} H(g_i(\reparam))
    \cdot p(\variables{u}) 
    \ \mathrm{d}\boldsymbol{u},
\end{align}
is equal to
\begin{align}
    \lim_{\coolness_1, \dots, \coolness_n\rightarrow +\infty}
    \int
    &\deriv \sigma(\coolness_j \cdot g_j(\reparam)) \cdot \\ \nonumber
    &\prod_{i \neq j} \sigma(\coolness_i \cdot g_i(\reparam))
    \cdot p(\variables{u}) 
    \ \mathrm{d}\boldsymbol{u}.
    \label{eq:sigmaterm}
\end{align}

For brevity's sake, we will write the products
\begin{align}
    \prod_{i \neq j} H(g_i(\reparam))
\end{align}
and
\begin{align}
    \prod_{i \neq j} \sigma(g_i(\reparam)),
\end{align}
as $\pi_j(\variables{u})$ and $\pi_j^\sigma(\variables{u})$, respectively.
Next, using distributional notation, Equation~\ref{eq:heavisideterm} can be further simplified as
\begin{align}
    &\left\langle \deriv (H \circ g_j \circ r),\ \pi_j \cdot p \right\rangle 
    = \\ \nonumber
    &\left\langle \delta \circ g_j \circ r,\ \deriv \left(g \circ r\right) \cdot \pi_j \cdot p \right\rangle.
\end{align}
Note that this expression utilises the assumption that $\deriv (g_j \circ r) \in L^1_{loc}(\mathbb{R}^k)$, i.e., $\deriv (g_j \circ r)$ is locally integrable over $\mathbb{R}^k$.
This asssumption is not very demanding, since distributions (generalised functions) are only well-defined when acting on functions that are at least locally integrable. Equation~\ref{eq:sigmaterm} can similarly be rewritten and simplified to obtain the equality
\begin{align}
    &\phantom{={}} \lim_{\coolness_1, \dots, \coolness_n\rightarrow +\infty} \left\langle \deriv (\sigma \circ g_j \circ r),\ \pi_j^\sigma \cdot p \right\rangle\\
    &= \nonumber
    \lim_{\coolness_1, \dots, \coolness_{j - 1}, \coolness_{j+1}, \dots, \coolness_n\rightarrow +\infty} \\
    \nonumber
    &\phantom{={}}
    \left\langle \delta \circ g_j \circ r,\ \deriv (g\circ r) \cdot \pi_j^\sigma \cdot p \right\rangle.
\end{align}
More explicitly,
\begin{align}
    (\text{\ref{eq:sigmaterm}})&=
    \lim_{\coolness_1, \dots, \coolness_n\rightarrow +\infty} \\ \nonumber
    &\phantom{={}}
    \int
    \deriv \sigma(\coolness_j \cdot g_j(\reparam)) \cdot 
    \pi_j^\sigma(\variables{u}) \cdot p(\variables{u})
    \ \differential \variables{u} \\
    &= \nonumber
    \lim_{\coolness_1, \dots, \coolness_n\rightarrow +\infty} \\ \nonumber
    &\phantom{={}}
    \int
    \frac{l\cdot e^{-g(\reparam)\cdot \coolness_j}}{(1 + e^{-g(\reparam)\cdot \coolness_j})^2}\cdot \\
    &\phantom{={}\int} \nonumber
    \deriv g_j(\reparam) \cdot
    \pi_j^\sigma(\variables{u}) \cdot p(\variables{u})
    \ \differential \variables{u} \\
    &= \nonumber
    \lim_{\coolness_1, \dots, \coolness_{j - 1}, \coolness_{j+1}, \dots, \coolness_n\rightarrow +\infty} \\ \nonumber
    &\phantom{={}}
    \int
    \delta(g_j(\reparam)) \cdot \\
    &\phantom{={}\int} \nonumber
    \deriv g_j(\reparam)\cdot
    \pi_j^\sigma(\variables{u}) \cdot p(\variables{u})
    \ \differential \variables{u}.
\end{align}
The last transition uses the fact that 
\begin{equation}
\lim_{\coolness_j \rightarrow +\infty} \frac{\coolness_j \cdot e^{-g(\reparam)\cdot \coolness_j}}{(1 + e^{-g(\reparam)\cdot \coolness_j})^2}
= \delta(g(\reparam)),
\end{equation}
in the distributional sense. In addition, we also have (again in distributional sense) that 
\begin{align}
\lim_{\coolness_i \rightarrow +\infty} \sigma(\coolness_i \cdot g_i(\reparam))
= H(g_i(\reparam)).
\end{align}

This final equation allows us to replace $\pi_j^\sigma(\variables{u})$ in the final line of Equation~\ref{eq:sigmaterm} with $\pi_j(\variables{u})$ by repeating the above steps for each index $i$ separately. Hence, we can conclude that our relaxation of $\deriv P(q)$ is indeed unbiased in the infinite coolness limit.
\end{prf}

\section{Experimental details}
\label{app:moreexperiments}

This section will give detailed \dspl programs, neural network architectures and elaborated figures for each of the experiments present in the main body of the paper. 
% All experiments were run on an HP ZBook Power G8 (NVIDIA T1200 GPU, Intel i9-11900H @ 2.50GHz, 16 GB RAM), except the LTN comparison in Section \ref{subsec:subtraction}.
All experiments were run on an RTX 3080 Ti coupled with a Intel Xeon Gold 6230R CPU @ 2.10GHz and 256 GB of RAM, except the LTN results of Section~\ref{subsec:dates}.
Note that the optimisation of any hyperparameters, such as learning rate or number of training epochs, was done via a grid search on a separate validation set.

\subsection{NeSy attention}
\label{app:exp_attention}

\paragraph{Setup details and \dspl program.}
The full \dspl program for the detection of handwritten years is given in Listing~\ref{program:dates}.
The query \probloginline{year} is optimised for a different number of samples depending on the experiment. For \expone, we have 28 000 training samples while there are 4000 validation and 8000 test samples. The set of years in the training, validation and test set are disjoint. For \exptwo, the size of validation and test set is the same as in the case of \expone, but with a training set of 40 000 samples. Here, the set of years in validation and test set are disjoint, but both are a subset of the set of years of the training set.

\begin{problogcode}
{\footnotesize
\begin{problog}
box(Params, B) ~ 
    generalisednormal(Params).
digit(Im, Loc, D) ~ 
    categorical(classifier([Im, Loc]), 
                [0, ..., 9]).
year(Im, Year1, Year2, Year3, Year4) :-
    region(Im, [Y1, Y2, Y3, Y4]),
    ordered_output([Y1, Y2, Y3, Y4]),
    box(Y1, B1), box(Y4, B4),
    x_diff(0.0, B1, B1diff), 
    B1diff < 0,
    x_diff(1.0, B4, B4diff), 
    0 < B4diff,
    digit(Im, Y1, D1), digit(Im, Y2, D2), 
    digit(Im, Y3, D3), digit(Im, Y4, D4),
    Year1 =:= D1, Year2 =:= D2), 
    Year3 =:= D3, Year4 =:= D4.

ordered_output([]).
ordered_output([[Mu, Sigma]]).
ordered_output([[Mu, Sigma], H2 | T]) :-
    box([Mu, Sigma], B1), box(H2, B2), 
    x_diff(B1, B2, Bdiff), 
    Bdiff < 0,
    ordered_output([H2 | T]).
\end{problog}
}
\caption{There is one continuous NDF, \probloginline{box}, which represents a bounding box as a generalised normal distribution with mean and scale being the center and width of the box, respectively. \probloginline{digit} is a discrete NDF that denotes the categorical distribution of the digit classifications made by the network \probloginline{classifier}. 
\probloginline{region} is the detection network that predicts the 4 bounding boxes, i.e., the parameters of four instances of \probloginline{box}.
Given these parameters, the predicate \probloginline{ordered_output} will enforce the spatial constraints that \probloginline{region} predicts its boxes in order from left to right on the image. It does so by taking the difference of the $x$ coordinate of each subsequent bounding box, which is a 2-dimensional random variable, and employing a \q{$<$} PCF.
Finally, the supervision on the digits of the year is given to the correct bounding box.
}
\label{program:dates}
\end{problogcode}

% {\footnotesize
% \begin{problog}
% box(Params, B) ~ generalisednormal(Params).
% digit(Im, Loc,  D) ~ categorical(classifier(Im, Loc), [0,...,9]).

% subtraction(Im, Diff, Dist) :-
%     object(Im, ID1), object(Im, ID2), ID1 =\= ID2, 
%     region(Im, ID1, y) =:= region(Im, ID2, y),
%     distance(Im, ID1, ID2, PredDist), PredDist =:= Dist, 
%     region(Im, ID1, x) < region(Im, ID2, x), 
%     Diff is digit(Im, ID1) - digit(Im, ID2).
    
% location_supervision(Im, X1, X2, Y1, Y2) :- 
%     object(Im, ID1), object(Im, ID2), ID1 =\= ID2,
%     region(Im, ID1, x) < region(Im, ID2, x),
%     region(Im, ID1, x) =:= X1 + width, region(Im, ID2, x) =:= X2 + width,
%     region(Im, ID1, y) =:= Y1 + width, region(Im, ID2, y) =:= Y2 + width.
    
% curriculum(Im, N1, N2, X1, X2, Y1, Y2 ) :-
%     object(Im, ID1), object(Im, ID2), ID1 =\= ID2,
%     region(Im, ID1, x) < region(Im, ID2, x),
%     digit(Im, ID1) =:= N1, digit(Im, ID2) =:= N2,
%     region(Im, ID1, x) =:= X1 + width, region(Im, ID2, x) =:= X2 + width,
%     region(Im, ID1, y) =:= Y1 + width, region(Im, ID2, y) =:= Y2 + width.
% \end{problog}
% }

\paragraph{Parameters and neural architectures.}
A schematic overview of the neural architecture used for all different methods can be seen in Figure \ref{fig:subtractionnet}.
The neural baseline simply outputs the four predictions of the classification network and optimises them by minimising the categorical cross-entropy on each digit of the year.
In the case of the neural-symbolic methods, the output of both the regression and classification components are used in the logic. \dspl optimises a binary cross-entropy on the probability of \probloginline{year}, while LTN optimises the MAX-sat objective function. 
As optimiser, we utilised Adamax~\citep{kingma2015adam} with its default learning rate of $10^{-3}$. \dspl and LTNs were run for 10 epochs, while the neural baseline was given 20 epochs, all with a batch size of 10. This number of epochs proved sufficient for all methods to converge.
Interestingly, no special annealing scheme was necessary for this experiment as constant value of $50$ for the coolness parameters lead to satisfactory results.
All these hyperparameters were determined through a grid search on the validation set. 

\begin{figure}
    \centering
    \includegraphics[width=\linewidth]{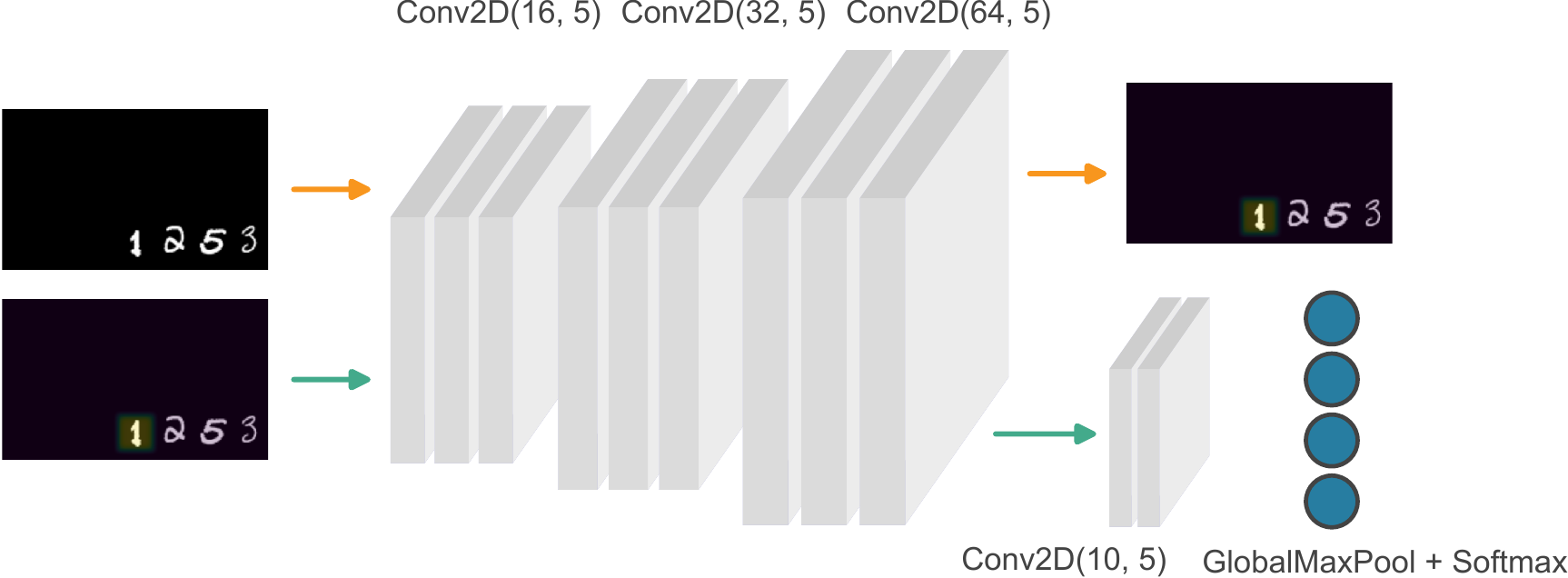}
    \caption{
    Overall neural architecture for the dates experiment. Following the \textcolor{yellow_orange}{orange} arrows first, the parameters of 4 generalised normal distributions are predicted for each image. Then, following the \textcolor{jungle_green}{green} arrows, the images is attenuated separately by each of the 4 distributions and then classified as a digit between 0 and 9 to give the total overall year prediction as an ordered tuple of 4 digits.
    }
    \label{fig:subtractionnet}
\end{figure}

\paragraph{Additional results and interpretations.}
% Figures \ref{fig:subtraction_samedistr_lossacc} and \ref{fig:subtraction_lossacc} show a more detailed evolution of the training process of the different methods. First, they illustrate the flexibility of NeSy methods since the pre-trained networks seemed to have little effect on the neural baseline. In particular, these networks only seem to affect the initial learning stages of the neural baseline, as can be seen from a couple of early peaks in validation accuracy. However, because of the lack of a proper function for the networks, a purely neural optimisation can not fully exploit these pre-trained states. Second, the evolution itself seems to be more consistent for both NeSy methods due to lower variability. Note however that, while LTNs can solve this tasks, they only provide a point estimate without any further indication to the uncertainty on this estimate. \dspl on the other hand models the location with inherent uncertainty. It also has to be mentioned that \dspl and LTNs are still a lot quicker than the neural baseline when looking at actual computation time. 
Roughly speaking, every 100 iterations took about 25 seconds for \dspl while the neural baseline took around 15 seconds. Given the results of Table~\ref{table:yearresults} and \dspl's satisfactory solution to the problem, the additional computational cost of adding probabilistic logic is worthwile in this case.

\subsection{Neural hybrid Bayesian network}
\label{app:morebayesnet}

\paragraph{Setup details and \dspl program.}
Our encoding of the neural hybrid Bayesian network of experiment~\ref{subsec:hybridnet} is given in Listing~\ref{program:hybridnet}. The goal is to optimise the neural networks responsible for the classification of \probloginline{humid} and \probloginline{cloudy} conditions, as well as the network that predicts the temperature value. Additionally, we explicitly model the noise present on the true temperature labels as a learnable parameter. To achieve this, a set of 1200 triples (\probloginline{Im1}, \probloginline{Im2}, \probloginline{X}) are used as training set, where \probloginline{Im1} is a CIFAR-10 image belonging to one of the first three classes, while \probloginline{Im2} belongs to the last two classes. In other words, we use CIFAR-10 images as proxies for real imagery data. \probloginline{X} is a set of 25 numerical meteorological features sampled from a publicly available Kaggle dataset~\citep{cho2020comparative}. The label of each triple is the probability that the weather, as described by the correct labels of \probloginline{humid}, \probloginline{cloudy} and \probloginline{temperature}, is good. Computing this probability label is non-trivial in itself. We utilised a large set of 1000 samples to approximate the correct underlying distributions and to obtain an approximate probability label.

\begin{problogcode}
{\footnotesize
\begin{problog}
humid(Im, H) ~ 
    bernoulli(humid_detector(Im)).
cloudy(Im, C) ~ 
    categorical(cloud_detector(Im), 
                [0, 1, 2]).

temperature(X, T) ~ 
    normal(temperature_detector(X), t(_)).
snowy_pleasant ~ beta(11, 7). 
rainy_pleasant ~ beta(1, 9)
cold_sunny_pleasant ~ beta(1, 1). 
warm_sunny_pleasant ~ beta(9, 2).

rainy(I1, I2) :-
    cloudy(I1, C), C =\= 0, 
    humid(I2, H), H =:= 1.

good_weather(I1, I2, X) :-
    rainy(I1, I2), 
    temp(X, T), T < 0, 
    snowy_pleasant > 0.5.
good_weather(I1, I2, X) :- 
    rainy(I1, I2), 
    temp(X, T), T >= 0, 
    rainy_pleasant > 0.5.
good_weather(I1, I2, X) :-
    \+rainy(I1, I2), 
    temp(X, T), T > 15, 
    warm_sunny_pleasant > 0.5.
good_weather(I1, I2, X) :-
    \+rainy(I1, I2), 
    temp(X, T), T <= 15, 
    cold_sunny_pleasant > 0.5.

P :: depressed(I1) :-
    cloudy(I1, C), C =:= N, 
    P is N * 0.2.

enjoy_weather(I1, I2, X) :-
    \+depressed(I1), 
    good_weather(I1, I2, X).
\end{problog}
}
\caption{
The NDFs \probloginline{humid} and \probloginline{cloudy} classify a given image as describing humid and cloudy conditions, respectively. \probloginline{temp} takes a set of 25 numerical features and predicts a mean temperature from those. Note that \probloginline{t(_)} is ProbLog notation for a single optimisable parameter. Depending on the value of the temperature, 4 different cases of weather and their degree of pleasantness are described by beta distributions. We define \probloginline{good_weather} as being true if the degree of pleasantness of any case is larger than 0.5. Finally, a person can be \probloginline{depressed} with probability 0.2 or 0.4 depending on the degree of \probloginline{cloudy}. Both then determine whether a person can enjoy the weather, if they are not \probloginline{depressed} and \probloginline{good_weather} is the case.
}
\label{program:hybridnet}
\end{problogcode}

\paragraph{Parameters and neural architectures.}
We utilise simple classifiers (Figure~\ref{fig:burglary_nets}) in the NDFs \probloginline{cloudy} and \probloginline{humid}, while the network in the neural predicate \probloginline{temperature} has three dense layers of size 35, 35 with ReLU activations and 1 with linear activation. Both classifiers share a common set of convolutional layers, requiring the learning of features that generalise to both classification problems. Additionally, the noise on the temperature prediction is modelled explicitly as a learnable TensorFlow variable with an initial value of 10. This choice is not arbitrary, as the initial neural parameter estimate will hover around the middle of the possible temperature values and a choice of 10 as initial standard deviation allows covering the entire range of temperature values with a non-insignificant probability mass. In this way, gradient information across the entire temperature domain can be accumulated during learning. Finally, \dspl was trained for 10 epochs using Adamax with learning rate $10^{-3}$ and batch size of 10.

\paragraph{Complications.}
Ideally, simple 0-1 labels of \probloginline{enjoy_weather} would be more intuitive, as we often do not observe the probability of an event but single cases where it is either true or false. However, our experiments have showed that our small dataset is insufficient to find an optimal solution using such labels in conjunction with the very distant supervision. To show that \dspl is still able to find solutions in cases where the supervision is slightly less distant using only 0-1 labels, we added a different neural hybrid Bayesian network experiment in Section~\ref{app:additionalexperiments} based on the well-known burglary-alarm example of probabilistic logic.

\paragraph{Additional results and interpretations.}
We want to stress that learning to predict the right mean temperature from the distant supervision is not straightforward. The only learning signal for the temperature has to pass through PCFs with a very wide range, meaning they do not specify the exact temperature value. Additionally, these PDFs still do not directly influence the supervision of \probloginline{enjoy_weather}, only \probloginline{good_weather}.
The Gaussian noise that renders the temperature into a continuous random variable only further convolutes the task. We conclude that \dspl can extract meaningful learning signals from reasonably distant supervision.

\subsection{Neural-symbolic variational autoencoder}
\label{app:morevae}

\paragraph{Setup details and \dspl programs.} 
Each data sample consists of 2 regular MNIST digits and the result of their subtraction. The first digit takes the place of the minuend while the second one is interpreted as the subtrahend. The training, validation and test sets had 30 000, 1 000 and 1 000 samples of this form, respectively. Encoding a VAE without additional logic in \dspl is straightforward (Listing \ref{program:regvae}), while adding logic involves more engineering freedom (Listing \ref{program:logicvae}). We opted for the simplest use of a conditional variational auto-encoder by only using the classified digit as additional input to the decoder. Note that during optimisation, both the VAE and digit classifier are trained jointly.

\begin{problogcode}
\begin{problog}
prior(P) ~ normal(0, 1).
latent(Im, L) ~ 
    normal(encoder_net(Im)).

good_image(Image) :- 
    prior(P), latent(Im, L),
    P =:= L,
    decoder_net(L, G),
    soft_unification(G, Image). 
\end{problog}
\caption{Prototypical implementation of a Gaussian VAE in \dspl. 
A normal prior \probloginline{prior} is used to regularise a Gaussian latent space modelled by the second NDF by expressing that they should be equal.
The decoder component of the VAE is given by \probloginline{decoder_net} and returns a generated image \probloginline{G} by sampling the latent space. This generation is self-supervised by soft unifying it with the given image.
}
\label{program:regvae}
\end{problogcode}

\begin{problogcode}
\begin{problog}
prior(ID, P) ~ normal(0, 1).
digit(Emb, D) ~ 
    categorical(digit_classifier(Emb),
                [0, ..., 9]).
latent(Im, L) ~ 
    normal(encoder_net(Im)).

good_subtraction(Im1, Im2, Diff) :- 
    prior(1, P1), prior(2, P2),
    latent(Im1, L1), latent(Im2, L2),
    L1 =:= P1, L2 =:= P2,
    embedding(Im1, E1), 
    embedding(Im2, E2),
    digit(E1, D1), digit(E2, D2),
    Diff =:= D1 - D2, 
    concat(L1, D1, ConditionalL1), 
    concat(L2, D2, ConditionalL2),
    decoder_net(ConditionalL1, G1),
    decoder_net(ConditionalL2, G2),
    soft_unification(G1, Image1),
    soft_unification(G1, Image1).
\end{problog}
\caption{Combining subtraction logic with a VAE in \dspl.
Each image is encoded into a Gaussian latent space and embedded into a lower-dimensional real space. The latent space is regularised by the standard normal prior while the embedding forms the input to a digit classifier to find which digit is on the image. The two classified digits, which follow a categorical distribution, should subtract to the given value of \probloginline{Diff}. Finally, the Gaussian latent space and the categorical digits are concatenated into the conditional latent space of the CVAE. The decoder network again samples from this space to construct a generation for both images, which should softly unify with the original images.
}
\label{program:logicvae}
\end{problogcode}

\paragraph{Parameters and neural architectures. } The NeSy VAE has two main neural components (Figure \ref{fig:vae}), one for the VAE itself and another that handles the digit classification used in the subtraction logic. A small set of 256 samples with direct supervision on the digit labels is used to pre-train the classification portion of the overall network to avoid degenerate solutions. All training utilised Adam as optimiser with a learning rate of $\cdot 10^{-3}$ and took 20 epochs using a batch size of 10. The pre-training was given 1 epoch with a batch size of 4.

\begin{figure}[ht]
    \centering
    \includegraphics[width=\linewidth]{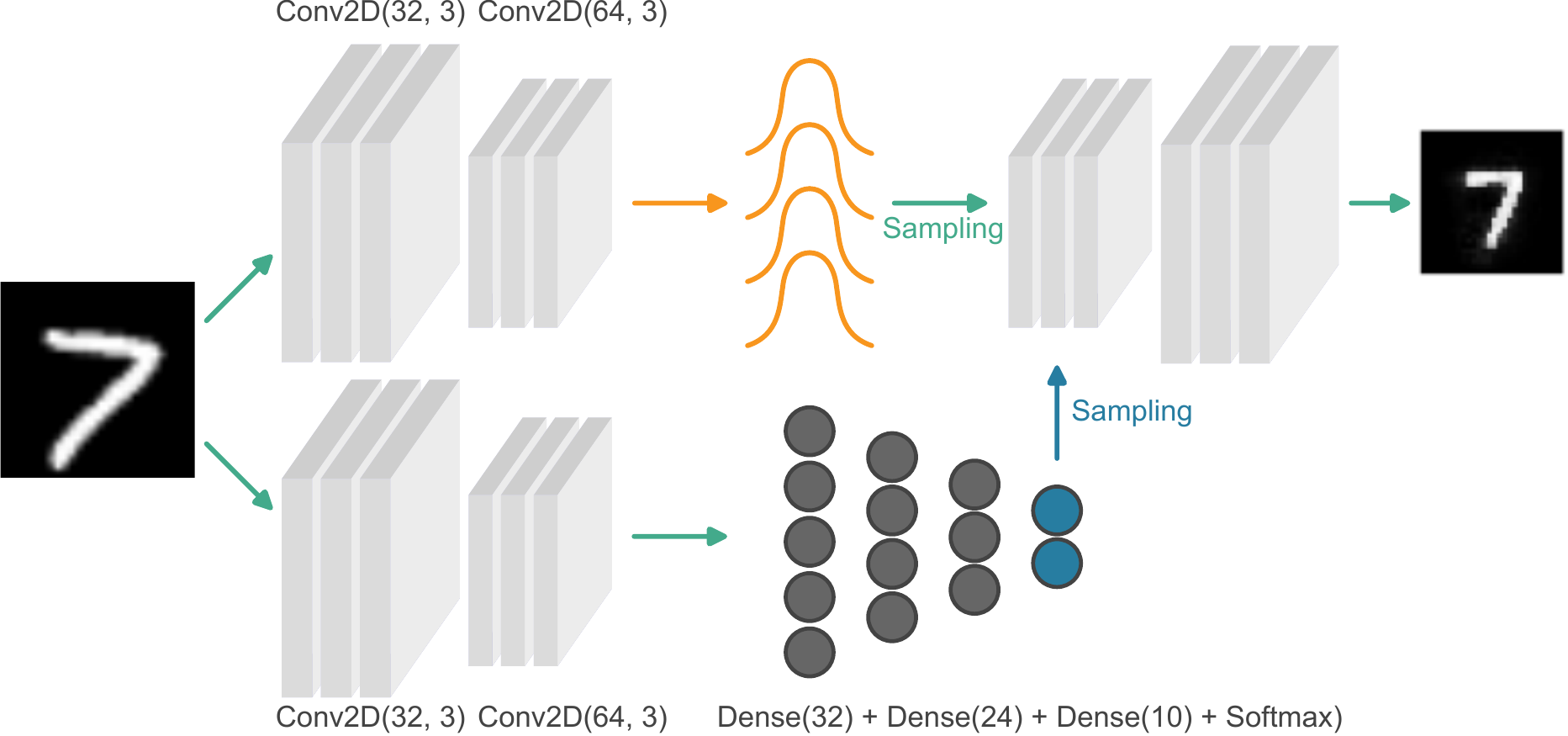}
    \caption{VAE encoder-decoder architecture. The decoder is equivalent to the transpose of the encoder. All layers use ReLU activation functions, except the final convolutional one, which applies a hyperbolic tangent.}
    \label{fig:vae}
\end{figure}

\paragraph{Complications.}
Regular Gaussian VAE optimisation has two components: a Kullback-Leibler (KL) divergence term and a reconstruction loss term. Since \dspl requires probabilistic values, i.e., between 0 and 1, a probabilistic translation of these terms is necessary for optimisation in \dspl. The KL divergence term compares the latent distribution of the VAE to a standard normal prior and can as such be replaced by a \probloginline{=:=} comparison in the logic. The reconstruction loss is chosen to be the exponentiation of a negated average $L_1$ loss function, as it yields a value between 0 and 1 that can be interpreted as the probability that two images match. Specifically, the loss between two such images $\boldsymbol{I_1}, \boldsymbol{I_2} \in \mathbb{R}^{768}$ is given by
\begin{equation}
    \exp(-\frac{1}{768}\sum_{i=1}^{768} \left|I_{1i} - I_{2i}\right|).
\end{equation}
The latter can be interpreted as a form of soft unification \citep{rocktaschel2017end}, which is why we denote it by the predicate \probloginline{soft_unification}.

\paragraph{Additional results and interpretations.}
Emphasis has to be put on the flexibility of generation in \dspl, as the generation of digits can be carried out in a range of different contexts without further optimisation. One only needs to write a query describing that logical context. The query that yields an image of both a left and right digit that subtract to a given value is given in Listing \ref{program:generative1}. The conditional query that generates an image of a right digit given an image of the left digit and their difference value is given in Listing \ref{program:generative2}.

\begin{problogcode}
\begin{problog}
generate_subtraction(G1, G2, Diff) :- 
    member(D1, [0, ..., 9]), 
    member(D2, [0, ..., 9]),
    prior(1, P1), prior(2, P2),
    Diff =:= D1 - D2, 
    concat(P1, D1, ConditionalL1), 
    concat(P2, D2, ConditionalL2),
    decoder_net(ConditionalL1, G1),
    decoder_net(ConditionalL2, G2).
\end{problog}
\caption{The logic finds all possible combinations for \probloginline{D1} and \probloginline{D2} that meet the subtraction evidence \probloginline{Diff} and concatenates these to a standard normal prior component into the conditional latent space. The decoder then generates images from a sample of this space.}
\label{program:generative1}
\end{problogcode}

\begin{problogcode}
\begin{problog}
generate_left(RightIm, Diff, LeftG) :-
    member(D1, [0, ..., 9]),
    embedding(RightIm, RightE),
    digit(RightE, RightD)
    Diff =:= LeftD - RightD,
    latent(RightIm, RightL),
    concat(RightL, LeftD, LeftCondL),
    decoder_net(LeftCondL1, LeftG).
\end{problog}
\caption{Given an image of the right digit and a difference value, we generate an image of a the left digit. The right's image is classified such that the logic can find the value of \probloginline{LeftD} that meets the given difference. By attaching that value to the Gaussian latent space of the right digit, the VAE can generate an image of the correct left digit in the \q{style} of the right one.}
\label{program:generative2}
\end{problogcode}

\section{Additional experiment}
\label{app:additionalexperiments}

An additional experiment was performed to show the promise of discrete-continuous neural probabilistic logic programming.
It is similar to the neural hybrid Bayesian network of Section~\ref{subsec:hybridnet}, but with more practical 1-0 query supervision.
% The second is a novel task similar the experiment of Section~\ref{subsec:dates}, but with additional PCFs over discrete random variables.

\subsection{Neural-continuous burglary alarm}
\label{app:morealarms}

\paragraph{Setup details and \dspl program.}
The neural-continuous burglary alarm (Listing \ref{program:alarm}) extends the classic example from Bayesian network literature (Listing \ref{program:classicalarm}). 

\begin{problogcode}
\begin{problog}
0.1 :: earthquake.
0.3 :: burglary.
0.9 :: hears.
    
0.7 :: alarm :- earthquake.
0.9 :: alarm :- burglary.

calls :- alarm, hears.
\end{problog}
\caption{Classical burglary-alarm ProbLog program. Three probabilistic facts \probloginline{earthquake}, \probloginline{burglary} and \probloginline{hears} are given with their probabilities. A neighbour calls when hearing an alarm, while an alarm can go off because of an earthquake or a burglary.}
\label{program:classicalarm}
\end{problogcode}

Each data sample is a triple $(E, B, L)$, where $E$ can be an MNIST digit 0, 1 or 2 while $B$ can be an MNIST 8 or 9. Values for $E$ of 0, 1 and 2 correspond to no earthquake, a mild earthquake or a heavy earthquake respectively. If $B$ is an MNIST 8, then there is no burglary. If it is 9, then there is a burglary. $L$ can have either the value 0 or 1, indicating whether the neighbour called or not. Our dataset contains 12 000 such triples for training, while having 1 000 for validation and 2000 for testing purposes. Obtaining the weak supervision $L$ is done by sampling according to the true probability of calling given the input To compute this true probability, a single sample is taken from the neighbour's true distribution. This true distribution has respective means of 6 and 3 for the horizontal and vertical Gaussian while both directions have a standard deviation of 3. Additionally, there are two possible ways to express that the distance of the neighbour should be smaller than 10 distance steps before hearing the alarm. One can use either the squared distance or the true distance in the rule \probloginline{hears}. A separation is often maintained in the weighted model integration literature \citep{zuidberg2019exact} between comparison formulae that are polynomial and those that are generally non-polynomial. To illustrate that \dspl can deal with both classes of formulae, we will perform experiments for both the squared distance (polynomial, Listing \ref{program:alarm}) and the true distance (non-polynomial, Listing \ref{program:truedistance}). Both these functions are implemented in Python and \dspl allows them to be easily imported as built-in predicates.

\begin{problogcode}
{\footnotesize
\begin{problog}
earthquake(Im, E) ~ 
    categorical(earthquake_net(Im), 
                [0, 1, 2]).
burglary(Im, B) ~ 
    categorical(burglary_net(Im), 
                [8, 9]).

neighbour(N) ~ 
    normal([t(@$\mu_x$@), t(@$\mu_y$@)], 
           [t(@$\sigma_x$@), t(@$\sigma_y$@)]).

hears :- 
    neighbour(N),
    squared_distance(0, N, D),
    D < 100.
    
P :: alarm(Im1, _) :- 
    earthquake(Im1, E), E =:= N, 
    P is N * 0.35.
0.9 :: alarm(_, Im2) :- 
    burglary(Im2, B), B =:= 9.

calls(Im1, Im2) :- 
    alarm(Im1, Im2), hears.
\end{problog}
}
\caption{
Our extension of the burglary alarm example has two categorical NDFs that model the chance of an earthquake and a burglary given an image. Additionally, whether the neighbour can hear the alarm if it goes off depends on their spatial distribution, which is modelled as a two-dimensional Gaussian distribution. This distribution is randomly initialised and its parameters need to be optimised.}
\label{program:alarm}
\end{problogcode}

\begin{problogcode}
\begin{problog}
hears :- 
    neighbour(N),
    distance(0, N, D),
    D < 10.
\end{problog}
\caption{Using the true distance in the \probloginline{hears} predicate as a case of a non-polynomial comparison formula.}
\label{program:truedistance}
\end{problogcode}

\paragraph{Parameters and neural architectures.} The complete neural architecture of both the earthquake and burglary classifiers is given in Figure \ref{fig:burglary_nets}. In addition to the neural parameters in these networks, four independent parameters are present in the program. These are used as the means and standard deviations for the neighbour's spatial distribution and are randomly initialised. Specifically, the means are sampled uniformly from the interval $\left[0, 10\right]$ while the standard deviations were sampled from $\left[2, 10\right]$. All optimisation was performed using regular stochastic gradient descent with a learning rate of $8\cdot 10^{-2}$ for two epochs using a batch size of 10.

\begin{figure}[ht]
    \centering
    \includegraphics[width=\linewidth]{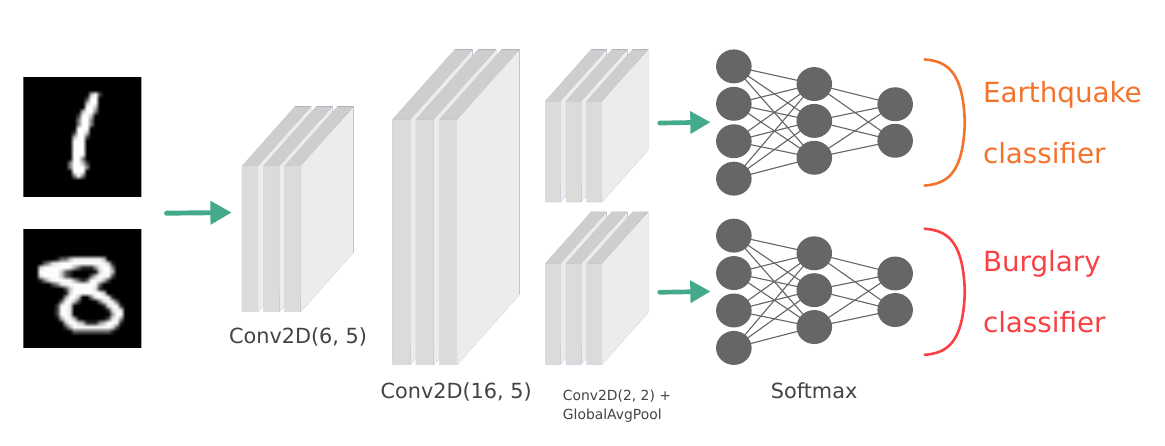}
    \caption{Overview of the architecture of the earthquake and burglary networks. Both share two convolutional layers, but each specific network applies its own final convolutional layer followed by a global average-pooling operation with softmax activation. All other activation functions are ReLUs.}
    \label{fig:burglary_nets}
\end{figure}

\paragraph{Complications.}
Because of the difference in nature between the parameters in the neural networks and the four independent parameters in the Gaussian distribution, the latter required a boosted learning rate to provide consistent convergence. Specifically, the gradients for these four parameters were multiplied by a value of 20, which was found by a hyperparameter optimisation on the validation set.

\paragraph{Results and interpretation.}
Initial learning progress of the neural networks seems volatile (Figure \ref{fig:burglary_discrete}), which is likely due to the unoptimised state of the neighbour's spatial distribution. Two epochs of training proves to be sufficient to optimise both the neural detectors and the neighbour's distribution. 
% In fact, the earthquake and burglary classifiers converge to respective test accuracies of $98.73_{-0.16}^{+0.22}$ and $98.43_{-0.50}^{+0.66}$ when using the squared distance and very similar results for the true distance. 
The 4 parameters of the neighbour's distribution do not converge to the true values, which is to be expected as their supervision is underspecified. However, they do converge to values that result in PCF probabilities that are close to those of the true underlying distribution. All in all, three conclusions can be drawn. First, this experiment indicates that \dspl is capable of jointly optimising neural parameters and independent, distributional parameters. Second, \dspl seems to be able to fully exploit both polynomial and more general non-polynomial comparison formulae. It shows the strength of our approximate approach, as exact methods often fail to efficiently deal with non-polynomial formulae \citep{zuidberg2019exact}. Third, \dspl can deduce meaningful probabilistic information from weak labels. Indeed, in order to optimise the neural detectors and the neighbour's distribution, \dspl has to aggregate meaningful update signals from the 0-1 labels across the given training data set to approximate the underlying probability of \probloginline{calls}. 

% To illustrate the strength of this final conclusion, consider the following. Assume that a burglary occurs and that the neural detector correctly classifies this occurrence, then the absolute difference in $P(\text{\probloginline{alarm(EarthquakeIm, BurglaryIm)}})$ between a mild or a heavy earthquake is only

% \begin{equation}
%     \left|0.9 + 0.35 - 0.9 \cdot 0.35 - (0.9 + 0.7 - 0.9 \cdot 0.7)\right| = 0.035,
% \end{equation}

% using Bayes' rule. Hence, a mild earthquake only has a very small effect on the overall probability, let alone in the case where the supervision itself is not even probabilistic.

\begin{figure}[ht]
    \centering
    \includegraphics[width=0.49\linewidth]{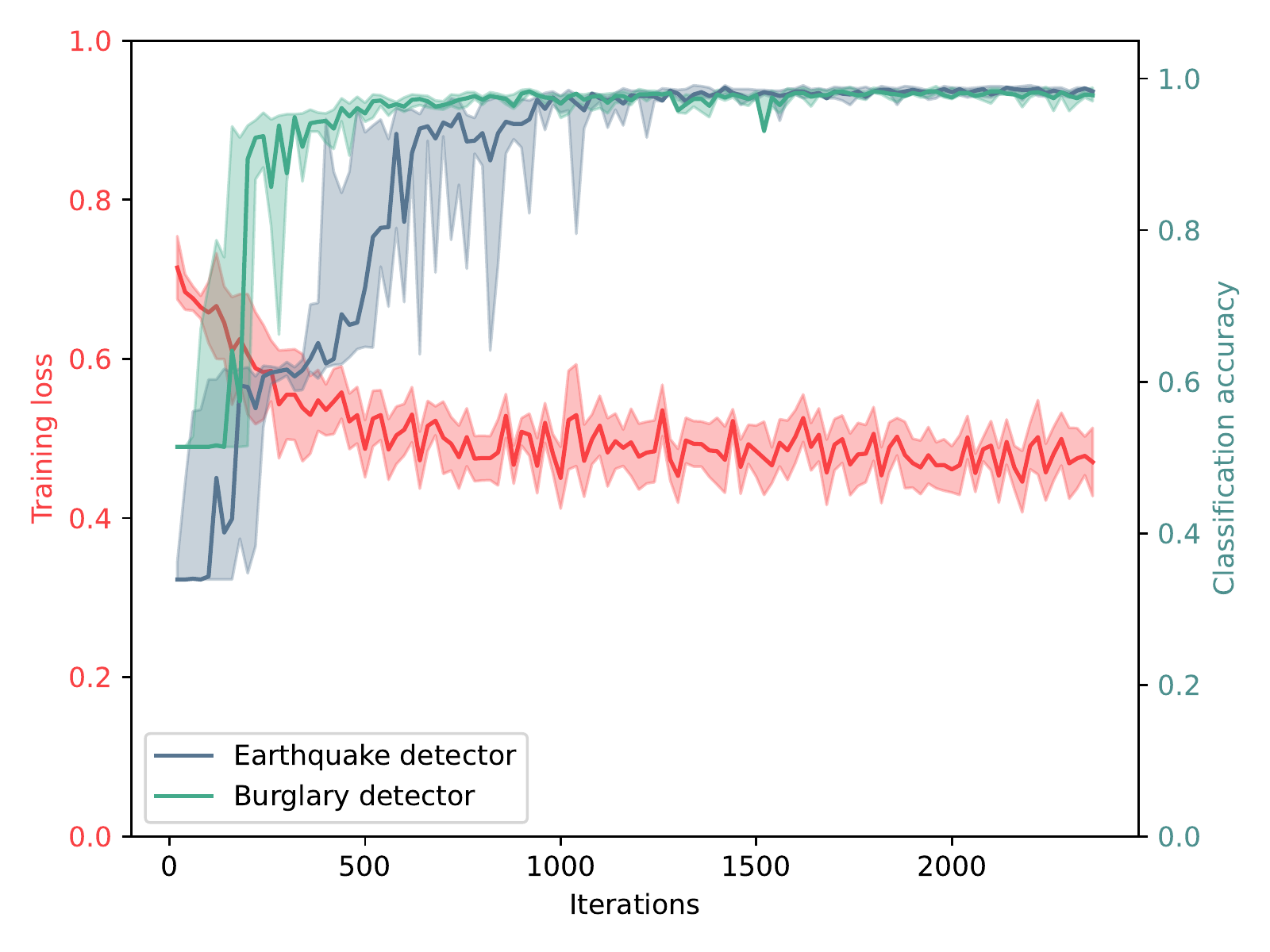}
    \includegraphics[width=0.49\linewidth]{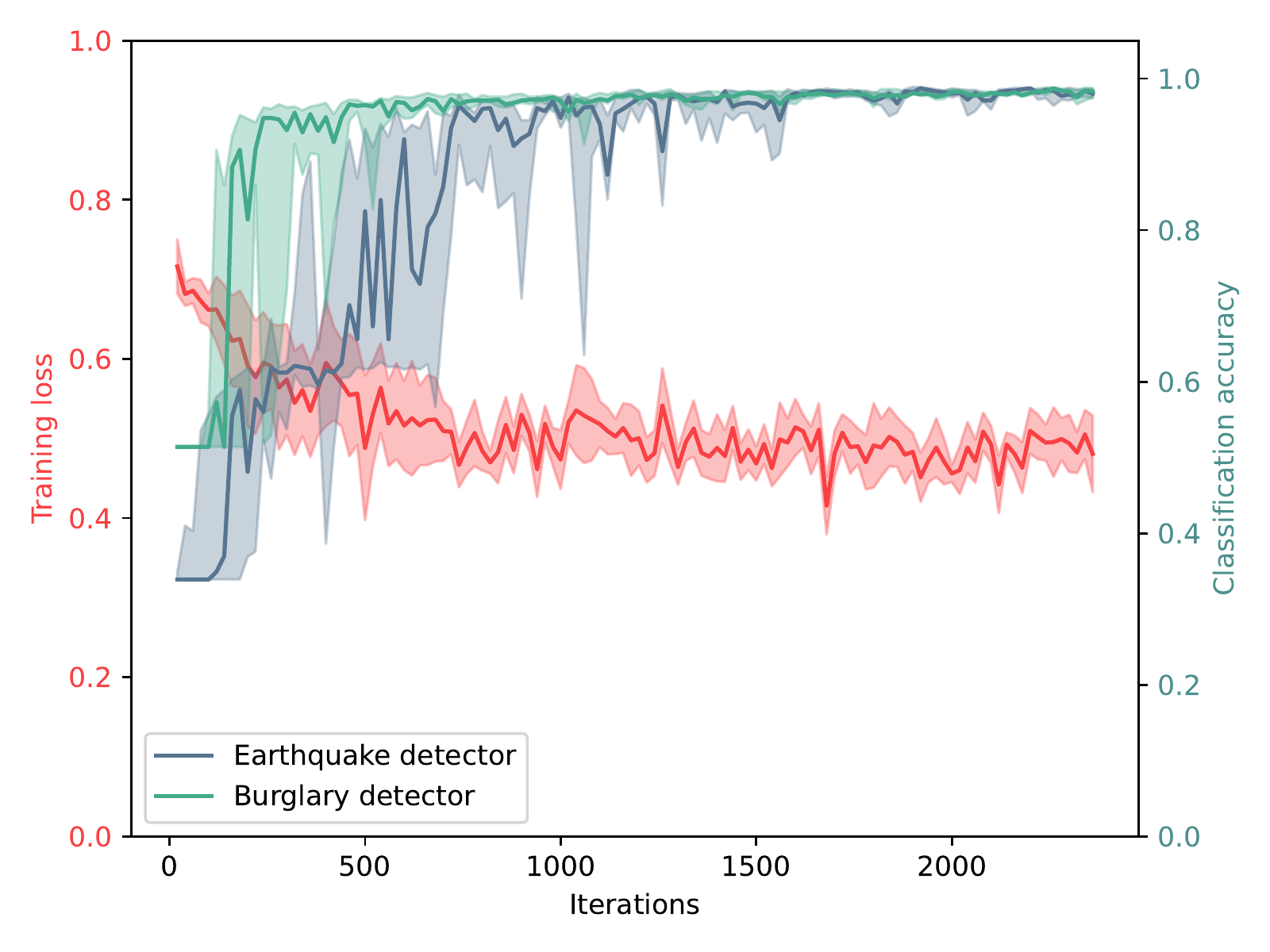}
    \caption{Evolution of the training loss and validation accuracy of the neural \q{earthquake} and \q{burglary} detectors. For both squared (left) and true distance (right), the discrete supervision seems to be sufficient to facilitate meaningful learning.}
    \label{fig:burglary_discrete}
\end{figure}

% \subsection{MNIST Subtraction}
% \label{app:subtraction}

% \input{experiment1.tex}

\section{Limitations}
\label{app:limitations}

The main limitation of \dspl is one that it inherits from probabilistic logic in general, computational tractability. Efficiently representing a probabilistic logic program is done via knowledge compilation, which is $\#P$-hard. Once the probabilistic program is knowledge compiled, evaluating the compiled structure is linear in the size of this structure.
Inference remains linear in the size of the compiled structure after the addition of continuous random variables as all samples can be run in parallel with the current inference algorithm.

Although our sampling strategy is efficient in the sense that it is linear in the number of samples, uses the advanced inference techniques of Tensorflow Probability to effectively sample higher dimensional distributions, and it can be executed in parallel for each sample, it remains ignorant of the comparison formulae that are approximated. More intricate inference strategies exist within the field of weighted model integration~\citep{morettin2021hybrid}, yet they currently lack the differentiability property to be integrated in \dspl's gradient-based optimisation. Conversely, our examples illustrate that our rather naive strategy is sufficient to solve basic tasks.
% Moreover, more intricate \begin{figure}[ht]
%     \centering
%     \includegraphics[width=0.32\linewidth]{Imagery/baseline_subtraction_diffdistr.pdf}
%     \includegraphics[width=0.32\linewidth]{Imagery/LTN_diffdistr_subtraction.pdf}
%     \includegraphics[width=0.32\linewidth]{Imagery/DCPL_diffdistr_subtraction.pdf}
%     \caption{Training evolution for the neural baseline (left) compard to LTNs (middle) and \dspl (right) in the out-of-distribution setting. The neural baseline fails outright to generalise to the validation and test set.}
%     \label{fig:subtraction_lossacc}
% \end{figure}sampling strategies do not always scale well under logical or algebraic constraints and so importance sampling techniques are still considered state-of-the-art~\citep{nitti2016probabilistic, tolpin2016design}. 
It is still an open question how to perform successful joint inference and gradient-based learning under general comparisons.

Orthogonal to the estimation of the integral during inference, exact knowledge compilation also prevents the scaling of \dspl to larger problem instances. Approximate knowledge compilation is the field of research that deals with tackling this issue. While it contains interesting recent work~\citep{fierens2015inference, huang2021scallop, manhaeve2021approximate}, it was highlighted by~\citeauthor{manhaeve2021approximate} that the introduction of the neural paradigm does lead to further complications. As such, we opted for exact knowledge compilation, but it has to be noted that we will be able to benefit from any future advances in the field of approximate inference. Alternatively, different semantics~\citep{winters2022deepstochlog} can simplify inference, but they lead to a degradation of expressivity of the language.

A potential future avenue for scaling up \dspl inference would be the use of further continuous relaxation schemes. More specifically, replacing discrete random variables with relaxed categorical variables~\citep{maddison2017concrete,jang2017categorical} might allow us, for instance, to forego the knowledge compilation step while still being able to pass around training signals.

\end{document}